\documentclass[10pt,twocolumn,letterpaper]{article}

\usepackage{cvpr}

\definecolor{cvprblue}{rgb}{0.21,0.49,0.74}
\usepackage[pagebackref,breaklinks,colorlinks,allcolors=cvprblue]{hyperref}
\usepackage{algorithm}
\usepackage{algpseudocode}
\usepackage{amsmath}
\usepackage{bm}
\usepackage{booktabs}
\usepackage{multirow}
\usepackage{colortbl}
\usepackage{xspace}
\usepackage{ntheorem}

\newtheorem{prop}{Proposition}
\newtheorem{lemma}{Lemma}

\newtheorem{ass}{Assumption}
\newtheorem{proof}{Proof}

\definecolor{aliceblue}{rgb}{0.94, 0.97, 1.0}
\definecolor{github}{rgb}{0.780,0.039,0.474}

\newcommand{\methodname}{OA-VAT\xspace}

\def\confName{CVPR}
\def\confYear{2026}

\title{Instance-level Visual Active Tracking with Occlusion-Aware Planning}

\author{
  \textbf{Haowei Sun}\textsuperscript{\rm 1}\thanks{Equal contribution. Email: sunhoward1105@gmail.com, kayjoe0723\\@gmail.com, hgao2729@gmail.com}~~
  \textbf{Kai Zhou}\textsuperscript{\rm 1}\footnotemark[1]~~
  \textbf{Hao Gao}\textsuperscript{\rm 1}\footnotemark[1]~~
  \textbf{Shiteng Zhang}\textsuperscript{\rm 1}~
  \textbf{Jinwu Hu}\textsuperscript{\rm 1 \rm 2}~\\
  \textbf{Xutao Wen}\textsuperscript{\rm 1}~
  \textbf{Qixiang Ye}\textsuperscript{\rm 4}~
  \textbf{Mingkui Tan}\textsuperscript{\rm 1 \rm 3}\footnotemark[2]\thanks{Corresponding author. Email: mingkuitan@scut.edu.cn}\\
  \textsuperscript{1} \small{South China University of Technology,}
  \textsuperscript{2} \small{Pazhou Laboratory,}
  \textsuperscript{3} \small{Key Laboratory of Big Data and Intelligent Robot, Ministry of Education,} \\
  \textsuperscript{4} \small{University of Chinese Academy of Sciences}\\
}

\begin{document}
\maketitle

\begin{abstract}
Visual Active Tracking (VAT) aims to control cameras to follow a target in 3D space, which is critical for applications like drone navigation and security surveillance. However, it faces two key bottlenecks in real-world deployment: confusion from visually similar distractors caused by insufficient instance-level discrimination and severe failure under occlusions due to the absence of active planning. To address these, we propose \textbf{\methodname}, a unified pipeline with three complementary modules. First, a training-free \textit{Instance-Aware Offline Prototype Initialization} aggregates multi-view augmented features via DINOv3 to construct discriminative instance prototypes, mitigating distractor confusion. Second, an Online Prototype Enhancement Tracker enhances prototypes online and integrates a confidence-aware Kalman filter for stable tracking under appearance and motion changes. Third, an Occlusion-Aware Trajectory Planner, trained on our new \texttt{Planning-20k} dataset, uses conditional diffusion to generate obstacle-avoiding paths for occlusion recovery. Experiments demonstrate \methodname achieves 0.93 average SR on UnrealCV (+2.2\% vs. SOTA TrackVLA), 90.8\% average CAR on real-world datasets (+12.1\% vs. SOTA GC-VAT), and 81.6\% TSR on a DJI Tello drone. Running at 35 FPS on an RTX 3090, it delivers robust, real-time performance for practical deployment. The code is available at \href{https://github.com/SHWplus/OA-VAT}{\color{github} \texttt{https://github.com/SHWplus/OA-VAT}}.
\end{abstract}

\begin{figure}[t]
  \centering
  \includegraphics[width=\linewidth]{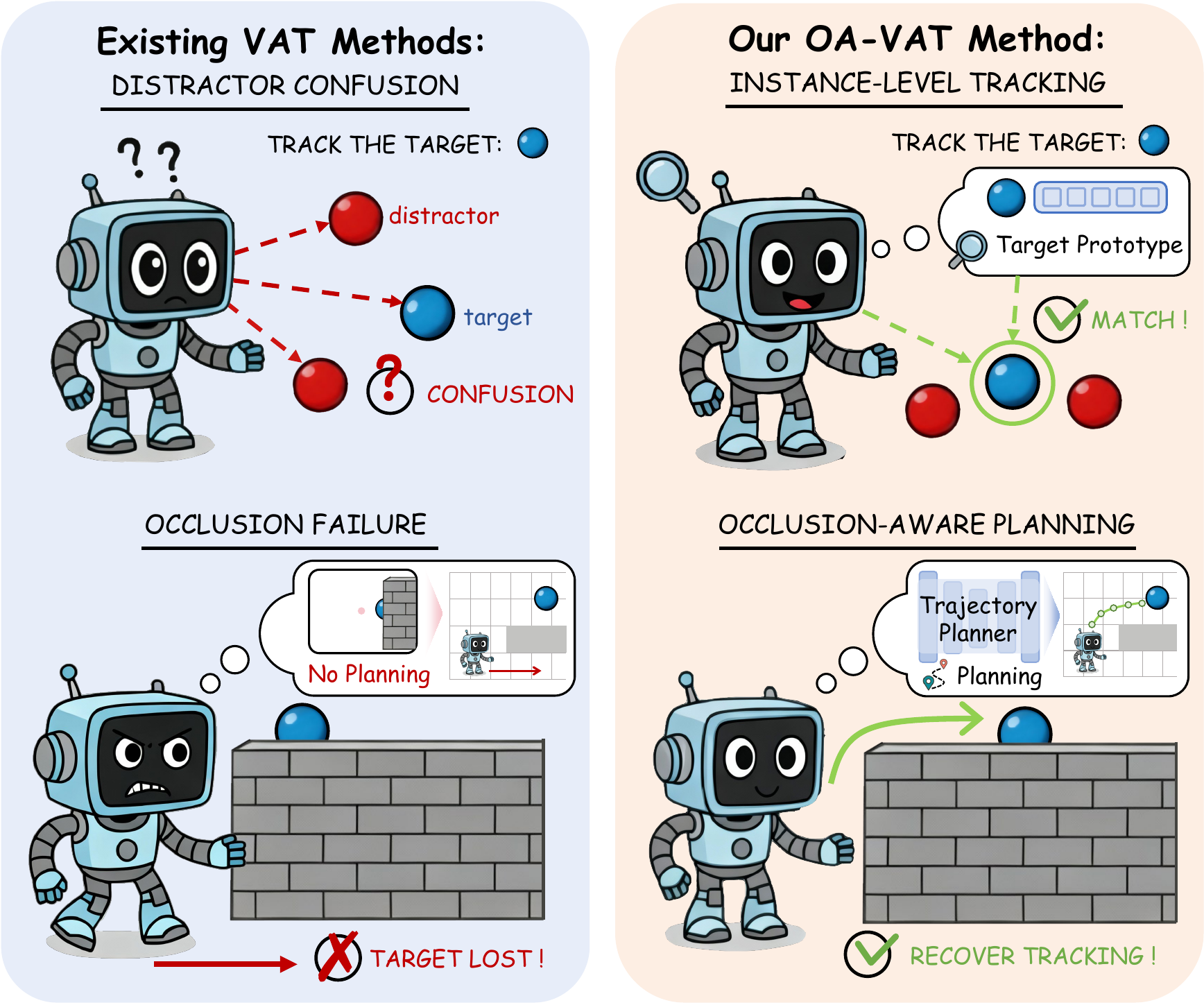}
    \caption{Comparison of existing VAT methods and \methodname. During offline initialization, we construct a target prototype to robustly match the target against distractors. During tracking, we employ a trajectory planner to recover the target under occlusions.}
  \label{fig:pipeline_compare}
  \vspace{-10pt}
\end{figure}

\begin{figure*}[t]
    \vspace{-8pt}
    \centering
    \includegraphics[width=\textwidth]{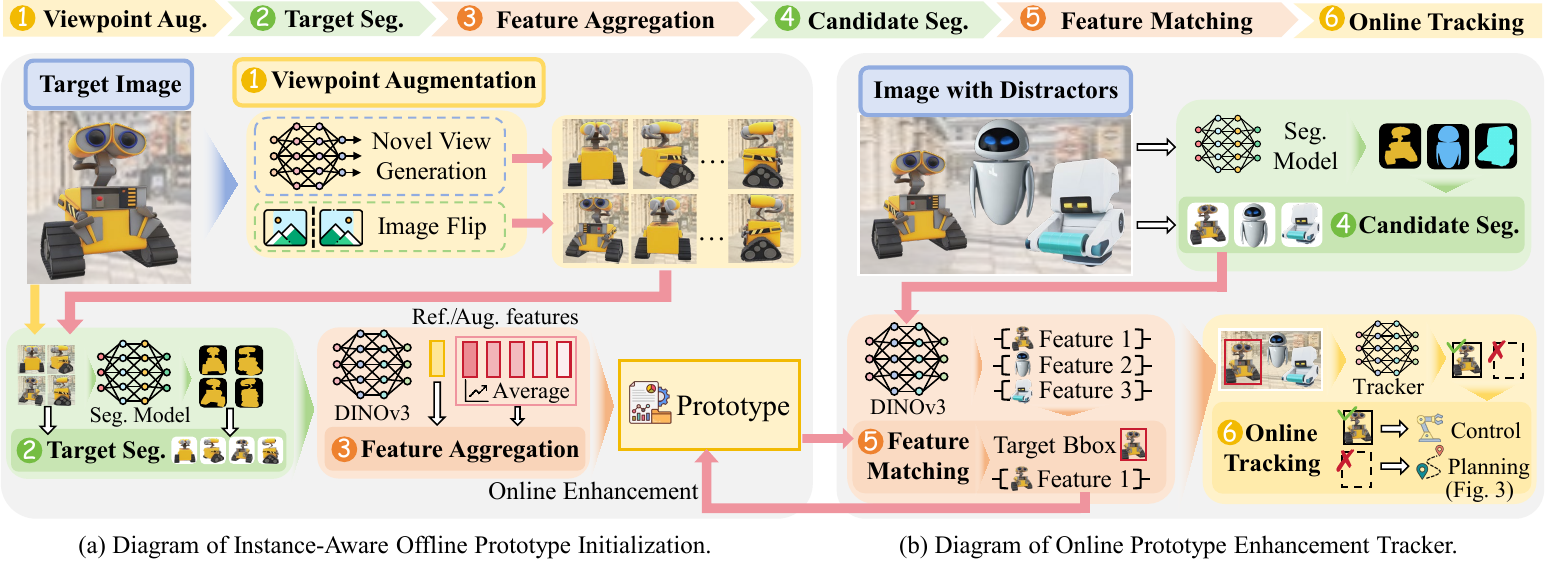}
    \caption{Overview of Occlusion-Aware-VAT (\methodname). (a) Given a reference image, \methodname first constructs an instance-aware prototype offline by aggregating features from viewpoint augmentations. (b) At runtime, \methodname first detects the target via matching prototype and then tracks it online with prototype enhancement. When tracking fails, it activates the planning module (Fig.~\ref{fig:planning}(b)) to recover target tracking.}
    \vspace{-12pt}
    \label{fig:prototype_apt}
\end{figure*}

\section{Introduction}
\label{sec:intro}

Visual Active Tracking (VAT) aims to dynamically control cameras to follow a specific target instance in 3D space \cite{dimp, sarl,advat,advat+,ts,rspt,evt,trackvla,dvat,gcvat}. It is widely used in real-world applications such as navigation \cite{liu2023aerialvln,wang2024realisticuavvisionlanguagenavigation,fan-etal-2023-aerial, MIR-2024-09-421} and security surveillance \cite{emran2018review, XING2022102972, ZHANG2023119243, Schedl2021AnAD}. Unlike passive visual tracking  \cite{babenko2009visual,ostrack, SiamRPN0,SORT,Chen_2023_CVPR,9573394,wen2021detection,smeulders2013visual} that operates on pre-recorded videos, VAT requires the agent to move the camera based on real-time visual observations. Passive visual tracking often falls short in the real world due to the dynamic nature of most targets. Thus, VAT offers a more practical yet challenging solution for real-world tracking applications.

Existing VAT approaches are broadly categorized into reinforcement learning (RL)-based \cite{sarl,advat+,evt,dvat,gcvat} and pipeline methods \cite{dimp,fan,rspt,trackvla}. RL-based methods learn end-to-end policies that map pixels to actions, enabling low-latency control without intermediate modules. However, they often struggle with sparse reward signals, leading to poor convergence in complex environments. Additionally, RL-based methods rely on simulation \cite{unrealcv,cvat,dvat,gcvat} for policy training, which limits their deployment due to the sim-to-real gap. In contrast, pipeline methods decouple tracking into separate perception and control stages. These methods use pre-trained visual models \cite{siamrpn,SiamFC,SiamMask,SiamRPN0,dinov2,yoloworld} for perception, enabling strong generalization to unseen environments. Consequently, pipeline VAT methods offer a more deployable alternative under real-world constraints.

Unfortunately, existing pipeline methods remain challenging in real scenarios, partly for the following reasons (Fig.~\ref{fig:pipeline_compare}). \textbf{1) Lack of instance-level tracking ability.} Real-world tracking often involves multiple similar distractors. However, most existing VAT methods \cite{dimp,fan,dvat,evt,gcvat} operate at the \textbf{category level} and struggle to distinguish specific target instances. \textbf{2) Missing active occlusion handling.} Most pipeline VAT methods \cite{fan,fasttracker2,das2018stable} employ simple controllers (\textit{e.g.}, PID \cite{pid}) that only center the target in the image, without considering occlusions. While simplifying system design, such controllers cannot navigate trackers around obstacles to recover the target, often causing failure.

To address these, we propose a pipeline VAT method called \methodname (see Fig.~\ref{fig:prototype_apt}), which enables instance-level discrimination and active recovery of occluded targets. First, we introduce an \textbf{Instance-Aware Offline Prototype Initialization} module. It extracts discriminative prototype features from the target's reference image. Second, we propose an \textbf{Online Prototype Enhancement Tracker} that tracks the target by matching the prototype to current frames.
It uses online prototype enhancement to handle target appearance changes and a confidence-aware Kalman filter for motion prediction.
Third, we develop an \textbf{Occlusion-Aware Trajectory Planner} to recover tracking during occlusions. Trained on our new \texttt{Planning-20k} dataset, it actively guides the camera around obstacles to restore target visibility. Our main contributions are as follows:

\begin{itemize}
    \item \textbf{A Robust Instance-Level Tracker.} We propose a training-free tracker that initializes a discriminative prototype from a reference image, enhances it online, and integrates a confidence-aware Kalman filter for robust tracking under appearance and motion changes.
    
    \item \textbf{An Active Occlusion-Aware Planner.} We develop a path planner trained on our new \texttt{Planning-20k} dataset that plans collision-free trajectories to recover occluded targets and generalizes to arbitrary unseen targets.
    \item \textbf{Extensive Validation.} Experiments on simulators, real-world images, and a drone demonstrate \methodname's state-of-the-art performance and real-time inference. 
\end{itemize}

\section{Task Definition}
\label{sec:task_definition}
Visual Active Tracking (VAT) aims to dynamically control a camera to follow a specific target instance in 3D space. The tracking agent must cope with challenges such as distractors, occlusions, and highly dynamic environments. We formulate VAT as a continuous visual control problem, characterized by the observation space, action space, and success criterion described below, and solved by a VAT tracker.

\textbf{Observation space $\mathcal{O}$.} At each time step, the agent receives an RGB image $\mathcal{I}_t$ (\textit{e.g.}, $160 \times 120$ pixels). Besides, a single reference image $\mathcal{I}_{ref}$, depicting the target instance to be tracked, is provided at the beginning of the episode.

\textbf{Action space $\mathcal{A}$.} The agent operates in a continuous action space $\mathcal{A} \in \mathbb{R}^4$, with $a_t=[v_f,v_l,v_v,\omega_y]^T$ representing linear velocities (forward, lateral, vertical) and yaw rotation.

\textbf{VAT Tracker.}  
We define a VAT tracker as an embodied agent parameterized by a policy $\pi_\theta$, which maps visual observations to continuous control actions, i.e.,
\begin{equation}
    a_t = \pi_\theta(\mathcal{I}_t, \mathcal{I}_{ref}).
    \label{eq:vat_tracker}
\end{equation}

\textbf{Success criterion.} We define a success criterion when the tracker keep the target instance centered in the image for a long duration. Metrics are detailed in Sec.~\ref{sec:exp_setting}.

\section{Handling Distractors and Occlusions in VAT}
\label{sec:method}

Visual active tracking in the real world is extremely challenging, primarily due to the prevalence of distractors and frequent occlusions. To address these challenges, we propose Occlusion-Aware VAT (\textbf{\methodname}), an instance-level tracking method with occlusion awareness, aiming to improve visual active tracking in complex real-world applications. As shown in Fig.~\ref{fig:prototype_apt}, \methodname incorporates a training-free instance tracking module that distinguishes the target from distractors. Moreover, we train a planning policy capable of recovering occluded targets. (see Fig.~\ref{fig:planning}).

\subsection{Instance-Aware Offline Prototype Initialization}
\label{sec:prototype_init}

We seek to build a training-free extraction module to obtain discriminative prototype features of the target instance. While existing visual foundation models (\textit{e.g.}, Grounding-DINO \cite{groundingdino}, SAM \cite{sam1}, DINOv3 \cite{dinov3}) excel at category-level detection and segmentation, they struggle to provide features that are sufficiently discriminative at the instance level. Consequently, directly representing the target with their raw features or masks often leads to confusion with distractors. To address this, we propose an offline module that extracts a robust prototype from foundation model outputs without any additional training. Furthermore, we provide theoretical analysis for its effectiveness in Sec.~\ref{sec:theory_analysis}.

Given a reference image $\mathcal{I}_{ref}$ of the target, we apply viewpoint augmentations and construct a prototype that captures view-invariant features for robust tracking. Specifically, we augment $\mathcal{I}_{ref}$ with horizontal and vertical flips to increase appearance diversity, producing an augmented image set $\{\mathcal{I}_i\}_i$. For human targets, whose appearance varies significantly across viewpoints, we generate additional views using an off-the-shelf diffusion model \cite{yuan2025poses} and augment them with the same flips. We then employ a segmentation model $\texttt{Seg}(\cdot)$ (we use YOLO-E \cite{yoloe} for its speed-performance balance) to obtain the target mask and crop the instance in the reference and augmented images. This produces a set of target crops $\tilde{\mathbf{I}} = \{\tilde{\mathcal{I}}_{ref},\tilde{\mathcal{I}}_1, ..., \tilde{\mathcal{I}}_N\}$, where $\tilde{\mathcal{I}}_{i}$ denotes the target crop of $\mathcal{I}_{i}$.

Subsequently, we employ a feature descriptor $\texttt{Desc}(\cdot)$, specifically DINOv3 \cite{dinov3} followed by global average pooling to extract features from the set $\tilde{\mathbf{I}}$:
\begin{equation}
    \mathbf{f}_{ref} = \texttt{Desc}(\tilde{\mathcal{I}}_{ref}), \  \mathbf{f}_i = \texttt{Desc}(\tilde{\mathcal{I}}_i), \ i=1,\ldots,N,
    \label{eq:dino_feature}
\end{equation}
where $\mathbf{f}_{ref}$ and $\mathbf{f}_i$ are the feature vectors for the reference and the $i$-th augmented image, respectively. The initial visual prototype $\tilde{\mathbf{f}}$ is obtained by averaging the original feature with the mean feature of the augmented set as follows:
\begin{equation}
    \tilde{\mathbf{f}} = \frac{ \mathbf{f}_{ref} + \frac{1}{N} \sum_{i=1}^N \mathbf{f}_i }{\lVert \mathbf{f}_{ref} + \frac{1}{N} \sum_{i=1}^N \mathbf{f}_i \rVert_2 }.
    \label{eq:feature_aggregate}
\end{equation}
The complete pipeline for the instance-aware offline prototype initialization module is shown in Fig.~\ref{fig:prototype_apt}(a).

\begin{algorithm}[t]
\caption{Online Prototype Enhancement Tracker}
\label{alg:online_tracking}
\begin{algorithmic}[1]
\Require Frame $\mathcal{I}_t$, prototype $\tilde{\mathbf{f}'}$, bounding box $\mathbf{b}_{t}$, seg. model $\texttt{Seg}(\cdot)$, feature descriptor $\texttt{Desc}(\cdot)$, similarity threshold $\eta_s$, confidence threshold $\eta_c$, tracker $\mathcal{T}$
\Ensure Updated prototype $\tilde{\mathbf{f}}'$, predicted box $\mathbf{b}_{t+1}$

\If{$\mathbf{b}_{t}=\varnothing$} 
\State $\%$ \textit{Detection via Prototype Matching}
\State Obtain $M$ candidate masks: $\{M^i\}_{i=1}^M \gets \texttt{Seg}(\mathcal{I}_t)$
\State Crop candidates: $\mathcal{I}_{\text{cand}}^i \gets \mathcal{I}_t \odot M^i, i \in \{1,\dots,M\}$
\State Extract features: $\mathbf{f}_{\text{cand}}^i \gets \texttt{Desc}(\mathcal{I}_{\text{cand}}^i)$

\State Compute cosine similarities $S$ via Eq.~\eqref{eq:similarity}

\State Find best candidate: $i^* \gets \arg\max_i S_i$
\State $\mathbf{b}_{t+1}\gets \varnothing$
\If{$S_{i^*} > \eta_s$}

\State $\mathbf{b}_{t+1}\gets$ bounding box of $\mathcal{I}_{\text{cand}}^{i^*}$
\EndIf

\Else 
\State $\%$ \textit{Online Tracking with Prototype Enhancement}
    \State Update tracker $\mathcal{T}$ with $\mathcal{I}_t$ and $\mathbf{b}_t$
    \State Get confidence $c_t$ and bounding box $\mathbf{z}_t$ from $\mathcal{T}$
    \State $\%$ \textit{Confidence-Aware Kalman Filter}
    \State Predict $\hat{\mathbf{x}}_{t|t-1}$ via Eq.~\eqref{eq:kf_state_update} and $\mathbf{b}_{t+1} \gets \mathbf{H}\hat{\mathbf{x}}_{t|t-1}$
    \If{$c_t<\eta_c$}
    \State $\mathbf{z}_t\gets\varnothing$
    \Else
        \State Update state $\hat{\mathbf{x}}_{t|t}$ via Eq.~\eqref{eq:kf_predict}
        \State Get predicted bounding box $\mathbf{b}_{t+1}\gets \mathbf{H}\hat{\mathbf{x}}_{t|t}$
        \State Extract normalized target feature $\hat{\mathbf{f}}_{\text{tar}}$
        \State Enhance Prototype $\tilde{\mathbf{f}'}$  via Eq.~\eqref{eq:ema}
\EndIf
\EndIf

\State \Return $\tilde{\mathbf{f}}'$, $\mathbf{b}_{t+1}$
\end{algorithmic}
\end{algorithm}

\subsection{Online Prototype Enhancement Tracker}
\label{sec:apt}

In the VAT setting, no initial bounding box is given, and the target’s appearance and motion vary greatly during tracking. To address this, we first detect the target by matching the initialized prototype to the current frame. During tracking, we enhance the prototype online using new frames and predict the target’s future motion to maintain robustness.

\begin{figure*}[t]
    \vspace{-8pt}
    \centering
    \includegraphics[width=\textwidth]{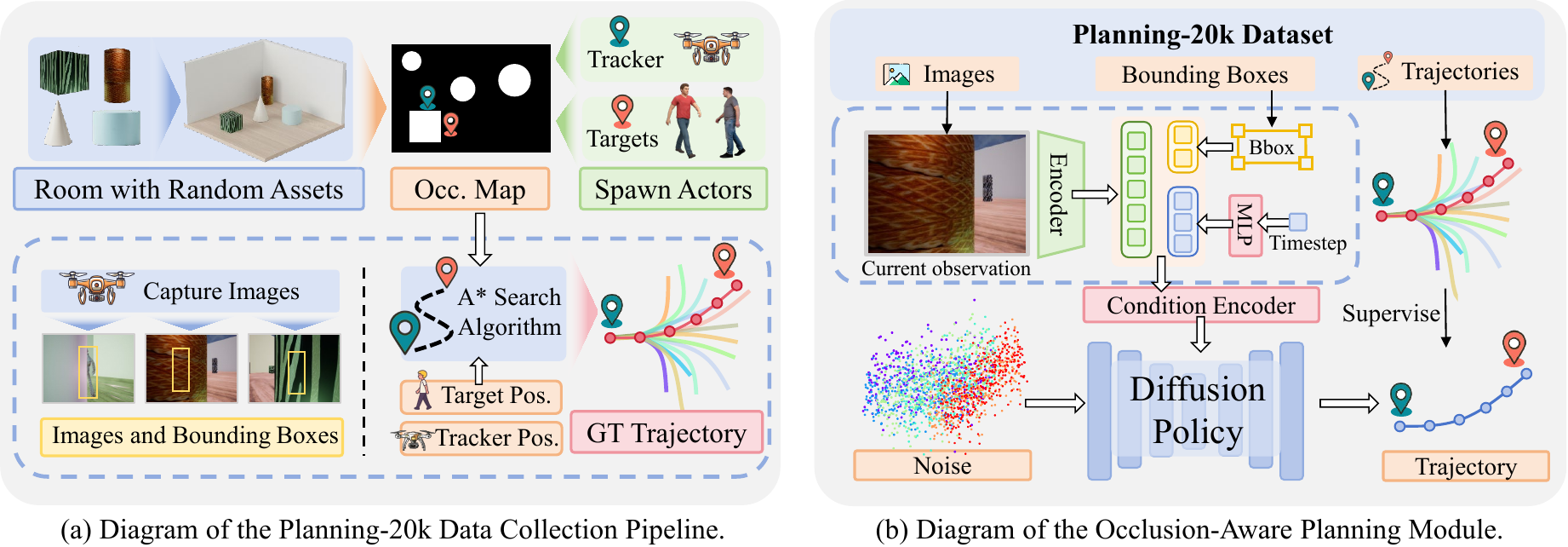}
    \caption{Overview of the proposed Occlusion-Aware Trajectory Planner. The planner denoises a random trajectory into a feasible recovery path conditioned on image observations and a predicted target bounding box, enabling target-agnostic trajectory planning.}
    \vspace{-12pt}
    \label{fig:planning}
\end{figure*}

\textbf{Online Visual Prototype Enhancement}. Our enhancement strategy operates based on the current tracking state. If the tracker is uninitialized (i.e., the current bounding box $\mathbf{b}_{t}=\varnothing$), we start a target detection procedure. For each incoming observation frame $\mathcal{I}_t$, we first employ the segmentation model $\texttt{Seg}(\cdot)$ to extract $M$ candidate object masks belonging to the target category. These candidates are then cropped from the frame, producing a set of image crops $\mathcal{I}_{\text{cand}} = \{\mathcal{I}_{\text{cand}}^1, ..., \mathcal{I}_{\text{cand}}^M\}$. Each crop is then processed by $\texttt{Desc}(\cdot)$ to obtain feature $\mathbf{f}_{\text{cand}}^i$ for the $i$-th candidate.

We then compute the cosine similarity between current visual prototype $\tilde{\mathbf{f}'}$ and each candidate feature $\mathbf{f}_{\text{cand}}^i$ as:
\begin{equation}
    S(\tilde{\mathbf{f}}', \mathbf{f}_{\text{cand}}^i) = \frac{\tilde{\mathbf{f}}' \cdot \mathbf{f}_{\text{cand}}^i}{\lVert\tilde{\mathbf{f}'}\rVert_2 \lVert\mathbf{f}_{\text{cand}}^i\rVert_2},
    \label{eq:similarity}
\end{equation}
where $\tilde{\mathbf{f}}'$ is initialized as $\tilde{\mathbf{f}}$.
The candidate with the highest similarity that exceeds a predefined threshold $\eta_s$ is identified as the target instance. Once a candidate is matched, its bounding box and the current frame $\mathcal{I}_t$ are used to initialize an object tracker \cite{wu2025ortrack} $\mathcal{T}$ for precise localization. 

During tracking, the tracker captures the target from varying viewpoints. These frames provide online features for enhancing the visual prototype. Specifically, we extract the normalized target feature $\hat{\mathbf{f}}_{\text{tar}}$ via $\texttt{Desc}(\cdot)$, and update the prototype via an exponential moving average (EMA):
\begin{equation}
    \tilde{\mathbf{f}}' \gets \beta \tilde{\mathbf{f}}' + (1-\beta) \hat{\mathbf{f}}_{\text{tar}},
    \label{eq:ema}
\end{equation}
where $\beta$ is the EMA momentum. To avoid slowing down the online tracking, the prototype is updated in a separate thread. This strategy ensures the prototype continuously adapts to the appearance variations, maintaining discriminability over long tracking sequences.

\textbf{Confidence-Aware Kalman Filter.} To handle rapidly changing target motions, we design a confidence-aware Kalman filter that adapts the filter parameters based on the tracker's confidence to ensure robust motion prediction. We define the state vector as $\mathbf{x}_t = [x, y, w, h, \dot{x}, \dot{y}, \dot{w}, \dot{h}]^T$, representing the bounding box and its derivatives. The standard Kalman filter \cite{kalman1960new} operates via a predict-update cycle, where the state prediction $\hat{\mathbf{x}}_{t\mid t-1}$ is given by: 
\begin{equation}
    \hat{\mathbf{x}}_{t\mid t-1}=\mathbf{F}\hat{\mathbf{x}}_{t-1\mid t-1},
    \label{eq:kf_state_update}
\end{equation}
where $\mathbf{F}$ is the state transition matrix. Given the current confidence $c_t$ and bounding box $\mathbf{z}_t$ from tracker $\mathcal{T}$, the state is updated when $c_t$ exceeds the threshold $\eta_c$:
\begin{equation}
    \hat{\mathbf{x}}_{t \mid t} = \hat{\mathbf{x}}_{t \mid t-1} + \mathbf{K}_t \left( \mathbf{z}_t - \mathbf{H} \hat{\mathbf{x}}_{t \mid t-1} \right),
    \label{eq:kf_predict}
\end{equation}
where $\mathbf{H}$ is the observation matrix and $\mathbf{K}_t$ is the Kalman gain, given by:
\begin{equation}
    \mathbf{K}_t = \mathbf{P}_{t \mid t-1} \mathbf{H}^\top \left( \mathbf{H} \mathbf{P}_{t \mid t-1} \mathbf{H}^\top + \mathbf{R}_t \right)^{-1},
    \label{eq:kf_gain}
\end{equation}
where $\mathbf{P}$ denotes the state covariance, and $\mathbf{R}_t$ is the measurement noise covariance. When the noise is large, the uncertainty in $\mathbf{z}_t$ increases, leading to a smaller $\mathbf{K}_t$, which assigns less weight to the current observation. The predicted bounding box can be obtained from $\hat{\mathbf{x}}_{t|t}$: $\mathbf{b}_{t+1}\gets \mathbf{H}\hat{\mathbf{x}}_{t|t}$.

Our confidence-aware Kalman filter adapts the noise $\mathbf{R}_t$ according to the confidence $c_t$. We hypothesize that $c_t$ reflects the uncertainty of the current observation. Thus, the noise variance $\sigma^2(\cdot)$ is modeled as a function of $c_t$:
\begin{equation}
    \mathbf{R}_t = \sigma^2(c_t) \mathbf{I}, \quad \sigma^2(c_t) = \frac{1}{1 + e^{\lambda \cdot (c_t - \gamma)}},
    \label{eq:cakf}
\end{equation}
where $\lambda$ and $\gamma$ are hyperparameters. This sigmoid mapping reduces $\sigma^2$ and increases trust in the observation when $c_t$ exceeds $\gamma$, while enlarging $\sigma^2$ when $c_t$ is low. Consequently, at low confidence, the Kalman gain $\mathbf{K}_t$ is suppressed, making the filter rely more on its internal state prediction than on the noisy observation. This mechanism also allows the filter to continue predicting the target motion as in Eq.~\eqref{eq:kf_state_update} during tracking failures, increasing the likelihood of re-acquiring the target. If no $\mathbf{z}_t$ is available for many consecutive frames, the planning module is triggered (see Sec.~\ref{sec:plan}). The pseudocode is shown in Algorithm~\ref{alg:online_tracking}.

\subsection{Occlusion-Aware Trajectory Planner}\label{sec:plan}

Another key challenge in open-world active tracking is frequent occlusion. Existing pipeline methods often use PID controllers \cite{pid} that steer the tracker with the image-plane distance between the target and the center. While effective for direct pursuit, these methods fail under occlusion, as they cannot plan trajectories to navigate around obstacles. 

To address this, we propose a planning module that generates recovery trajectories via imitation learning. Since expert demonstrations under occlusion are unavailable, we introduce $\texttt{Planning-20k}$, a synthetic dataset collected in UnrealCV~\cite{unrealcv}. Using this data, we train a diffusion model to predict tracker trajectories conditioned on visual observations and the target bounding box. This design makes the policy inherently \textit{target-agnostic}, enabling zero-shot generalization to arbitrary unseen targets.

\textbf{Data Collection.} The design of our \texttt{Planning-20k} dataset is central to learning a target-agnostic occlusion-recovery policy. As illustrated in Fig.~\ref{fig:planning}(a), we generate data in UnrealCV’s \texttt{SimpleRoom} environment as follows.

We first construct each map by randomly placing obstacles in an empty room and building a 2D occupancy map. To simulate occlusions, we randomly select an obstacle and spawn the target on one of its bounding box edges $e_i$. The tracker is then placed on an adjacent edge to $e_i$ to capture an RGB image and the target bounding box. Samples where the target is fully visible are discarded, ensuring only non-trivial planning scenarios. To enhance visual diversity, we apply domain randomization to illumination and obstacle textures using a texture dataset~\cite{texture}. Expert trajectories are then obtained via A* search algorithm \cite{a*} on the occupancy map. The final dataset contains 20k samples, including 8k with default textures and 12k with randomized ones. Each sample consists of an RGB image, a target bounding box, and the corresponding expert trajectory.

{\bf Data Diversity.} Our \texttt{Planning-20k} covers common occlusion structures in real-world scenarios. \textbf{(1)} Single-side occlusion blocks one side of the view, requiring the planner to go around it. \textbf{(2)} Double-side occlusion involves occlusions on both sides, leaving only a central gap for precise navigation. \textbf{(3)} Corridor-type occlusion includes obstacles on the left, right, and rear side, mimicking entry into a narrow hallway from a room.

\textbf{Occlusion-Aware Planning Module.} To achieve robust trajectory planning under occlusions, we propose a planning module inspired by Diffusion Policy \cite{chi2024diffusionpolicy}. We formulate the trajectory planning problem as a conditional denoising diffusion process. Previous diffusion policies \cite{chi2024diffusionpolicy} conditioned solely on visual observations often overfit the visual appearance of specific targets, limiting their generalization when the target changes. \textit{In contrast}, \methodname incorporates target bounding boxes as an explicit condition. This guides the model to learn \textit{target-agnostic} trajectory planning, thereby enabling robust generalization to unseen targets.

Formally, let $\mathbf{A}_t$ denotes trajectory points over a horizon $T_p$. We model the conditional distribution $p(\mathbf{A}_t | \mathcal{I}_t, \mathbf{b}_t)$ where $\mathcal{I}_t$ is the image and $\mathbf{b}_t$ is the target bounding box. The diffusion process starts from noise $\mathbf{A}^K_t \sim \mathcal{N}(0, \mathbf{I})$ and iteratively denoises for $K$ steps via:
\begin{equation}
    \mathbf{A}^{k-1}_t = \alpha \left( \mathbf{A}^k_t - \phi \epsilon_\theta(\mathcal{I}_t, \mathbf{b}_t, \mathbf{A}^k_t, k) \right) + \mathcal{N}(0, \sigma^2 \mathbf{I}),
    \label{eq:dp_denoise}
\end{equation}
where $\epsilon_\theta$ is a noise prediction network \cite{unet} that estimates the gradient of the action score function, and $\alpha, \phi, \sigma$ are noise scheduling function of iteration $k$~\cite{nichol2021improved}. We use the mean squared error (MSE) to define the training objective:
\begin{equation}
    \mathcal{L} = \mathbb{E}_{k, \mathbf{A}^0_t, \bm{\varepsilon}} \left[ \| \bm{\varepsilon} - \epsilon_\theta(\mathcal{I}_t, \mathbf{b}_t, \mathbf{A}^0_t + \bm{\varepsilon}, k) \|^2 \right],
    \label{eq:dp_loss}
\end{equation}
where $\epsilon$ is the random noise. $\mathcal{L}$ encourages the network to reconstruct the noise added to the ground-truth trajectory.

During online tracking, when the tracker (in Sec.~\ref{sec:apt}) becomes unreliable, \textit{i.e.}, $c_t < \eta_c$, we employ the confidence-aware Kalman filter in Eq.~\eqref{eq:kf_state_update} to predict the bounding box $\mathbf{b}_t$ of the occluded target. $\mathbf{b}_t$ is then used as the conditioning input to our planner, generating a recovery trajectory. 

{\bf Why our planner generalizes well.} Our planner achieves robust generalization by modeling occlusions with physical rules. Specifically, it infers navigable pathways from spatial relationships between targets and obstacles. Moreover, we use bounding boxes as conditions, ensuring the planner focuses on target locations rather than visual textures. This captures general physical rules, and avoids reliance on large-scale photorealistic data.

\subsection{Theoretical Guarantees on Instance Prototype} 
\label{sec:theory_analysis}
For any target instance $T_k$, we define its true feature manifold $M_k$ as the set of all features extracted by $\texttt{Desc}(\cdot)$ under arbitrary imaging conditions. From this manifold, the normalized reference features of $T_k$ are obtained via Eq.~\eqref{eq:dino_feature} and denoted by the set $F^*_k$. For each reference feature $f^*_k \in F^*_k$, we generate a corresponding set of multi-view augmented features $\{f_{k,i}\}_{i=1}^N$. The instance-aware prototypes are then derived via Eq.~\eqref{eq:feature_aggregate} and form the set $\hat{F}_k$. For any manifold $M$, $\mathbb{E}_{g \sim M}[\cdot]$ denotes expectation over $M$.

\begin{prop}
\label{prop:prototype}
Under the assumptions that (i) multi-view augmented features $\{f_{k,i}\}_{i=1}^N$ better cover the true feature manifold $M_k$ than the reference feature $f^*_k$ and (ii) features from the same target are cohesive while those from different targets are well-separated, for any two distinct targets $T_k \neq T_j$, the minimum squared distance between any pair of $(\hat{f}_k, \hat{f}_j)$ sampled from $\hat{F}_k$ and $\hat{F}_j$ is larger than that between any pair of $(f^*_k, f^*_j)$ sampled from $F^*_k$ and $F^*_j$:
\begin{equation}
\min_{\hat{f}_k \in \hat{F}_k, \hat{f}_j \in \hat{F}_j} \|\hat{f}_k - \hat{f}_j\|_2^2 \geq \min_{f^*_k \in F^*_k, f^*_j \in F^*_j} \|f^*_k - f^*_j\|_2^2.
\end{equation}
\end{prop}

\textit{For the proof of Proposition~\ref{prop:prototype}, please refer to the Appendix.} 
Proposition~\ref{prop:prototype} shows that the instance-aware offline prototype initialization module (Sec.~\ref{sec:prototype_init}) improves inter-instance separation by aggregating multi-view features, producing prototypes more discriminative than the original foundation model features, as validated by Fig.~\ref{fig:ablation_prototype_init}.

\begin{table*}[t]
    \small
    \centering
        \caption{Results in UnrealCV environments containing \textbf{distractors}. \textbf{Bold} represents the best while \underline{underline} represents the second.  TrackVLA~\cite{trackvla} does not report $AR$ and its training code and checkpoints are not publicly available, so we cannot reproduce this metric.}
\renewcommand{\arraystretch}{0.98}
\renewcommand{\tabcolsep}{3.5pt}
    \begin{tabular}{l|c|ccc|ccc|ccc|ccc|r}
        \toprule
        \multirow{2}{*}{Tracker} & \multirow{2}{*}{Publication} & \multicolumn{3}{c|}{Parking Lot (2D)} & \multicolumn{3}{c|}{UrbanCity (4D)} & \multicolumn{3}{c|}{ComplexRoom (4D)} & \multicolumn{3}{c|}{Average} & \multirow{2}{*}{Params.}\\ 
         & & $AR\uparrow$ & $EL\uparrow$ & $SR\uparrow$ & $AR\uparrow$ & $EL\uparrow$ & $SR\uparrow$ & $AR\uparrow$ & $EL\uparrow$ & $SR\uparrow$ & $AR\uparrow$ & $EL\uparrow$ & $SR\uparrow$ \\ \midrule
         DiMP \cite{dimp} & ICCV 2019 & 111 & 271 & 0.24 & 170 & 348 & 0.32 & 97 & 307 & 0.26 & 126 & 309 & 0.27 & $26M$ \\
         SARL \cite{sarl} & TPAMI 2019 & 53 & 237 & 0.12 & 74 & 221 & 0.16 & 22 & 263 & 0.15 & 50 & 240 & 0.14 & $2M$ \\
         AD-VAT \cite{advat} & ICLR 2019 & 43 & 232 & 0.13 & 32 & 204 & 0.06 & 16 & 223 & 0.16 & 30 & 220 & 0.12 & $4M$ \\
         AD-VAT+ \cite{advat+} & TPAMI 2019 & 35 & 166 & 0.08 & 89 & 245 & 0.11 & 35 & 262 & 0.18 & 53 & 224 & 0.12 & $4M$ \\
         TS \cite{ts} & ICML 2021 & 186 & 331 & 0.39 & 227 & 381 & 0.51 & 250 & 401 & 0.54 & 221 & 371 & 0.48 & $7M$ \\
         EVT \cite{evt} & ECCV 2024 & \underline{192} & 425 & 0.63 & \underline{272} & 472 & \underline{0.92} & \underline{354} & \underline{479} & 0.88 & \underline{273} & 459 & 0.81 & $748M$\\
         FAn \cite{fan} & RAL 2024 & 126 & 301 & 0.28 & 167 & 334 & 0.30 & 189 & 351 & 0.29 & 161 & 329 & 0.29 & $132M$\\
         FAn+SAM2 \cite{sam2} & ICLR 2025 & 170 & 349 & 0.40 & 201 & 407 & 0.55 & 262 & 422 & 0.61 & 211 & 393 & 0.52 & $122M$\\
         TrackVLA \cite{trackvla} & CoRL 2025 & - & \underline{467} & \underline{0.90} & - & \underline{476} & \underline{0.92} & - & \underline{479} & \underline{0.91} & - & \underline{474} & \underline{0.91} & $>7B$\\
         \rowcolor{aliceblue} \textbf{Ours} & CVPR 2026 & \textbf{392} & \textbf{482} & \textbf{0.93} & \textbf{385} & \textbf{486} & \textbf{0.95} & \textbf{392} & \textbf{481} & \textbf{0.92} & \textbf{390} & \textbf{483} & \textbf{0.93} & $584M$ \\
         \bottomrule
    \end{tabular}
    \label{tab:unrealcv_hard}
\end{table*}

\begin{table*}[t]
    \small
    \centering
    \begin{minipage}{.70\textwidth}
    \centering
        \caption{Performance on DAT benchmark under within scene setting.}
\renewcommand{\arraystretch}{0.98}
\renewcommand{\tabcolsep}{2.6pt}
    \begin{tabular}{l|c|ll|ll|ll}
        \toprule
        \multirow{2}{*}{Tracker} & \multirow{2}{*}{Publication} & \multicolumn{2}{c|}{citystreet} & \multicolumn{2}{c|}{desert} & \multicolumn{2}{c}{village} \\ 
         & & $CR\uparrow$ & $TSR\uparrow$ & $CR\uparrow$ & $TSR\uparrow$  & $CR\uparrow$ & $TSR\uparrow$ \\ \midrule
         SARL \cite{sarl} & TPAMI 2019 & $49_{\pm3}$ & $0.25_{\pm0.02}$ & $9_{\pm1}$ & $0.06_{\pm0.00}$ & $46_{\pm5}$ & $0.23_{\pm0.03}$ \\
         D-VAT \cite{dvat} & RAL 2024 & $48_{\pm8}$ & $0.26_{\pm0.02}$ & $47_{\pm13}$ & $0.26_{\pm0.04}$ & $44_{\pm8}$ & $0.22_{\pm0.05}$\\
         GC-VAT \cite{gcvat} & NeurIPS 2025 & $\underline{279}_{\pm110}$ & $\underline{0.80}_{\pm0.30}$ & $\underline{307}_{\pm124}$ & $\textbf{0.84}_{\pm0.29}$ & $\underline{239}_{\pm134}$ & $\underline{0.73}_{\pm0.32}$\\
         \rowcolor{aliceblue} \textbf{Ours} & CVPR 2026 & $\textbf{310}_{\pm2}$ & $\textbf{0.83}_{\pm0.01}$ & $\textbf{311}_{\pm3}$ & $\underline{0.83}_{\pm0.01}$ & $\textbf{307}_{\pm2}$ & $\textbf{0.83}_{\pm0.01}$\\
         \midrule
         \multirow{2}{*}{Tracker} & \multirow{2}{*}{Publication} & \multicolumn{2}{c|}{downtown} & \multicolumn{2}{c|}{lake} & \multicolumn{2}{c}{farmland} \\
         & & $CR\uparrow$ & $TSR\uparrow$ & $CR\uparrow$ & $TSR\uparrow$  & $CR\uparrow$ & $TSR\uparrow$ \\ \midrule
         SARL \cite{sarl} & TPAMI 2019 & $54_{\pm5}$ & $0.29_{\pm0.01}$ & $47_{\pm3}$ & $0.24_{\pm0.02}$ & $60_{\pm25}$ & $0.23_{\pm0.01}$\\
         D-VAT \cite{dvat} & RAL 2024 & $9_{\pm1}$ & $0.06_{\pm0.01}$ & $46_{\pm8}$ & $0.26_{\pm0.06}$ & $13_{\pm1}$ & $0.07_{\pm0.00}$\\
         GC-VAT \cite{gcvat} & NeurIPS 2025 & $\underline{203}_{\pm119}$ & $\underline{0.65}_{\pm0.30}$ & $\underline{181}_{\pm116}$ & $\underline{0.61}_{\pm0.31}$ & $\underline{243}_{\pm117}$ & $\underline{0.68}_{\pm0.32}$\\
         \rowcolor{aliceblue} \textbf{Ours} & CVPR 2026 & $\textbf{370}_{\pm3}$ & $\textbf{0.99}_{\pm0.01}$ & $\textbf{318}_{\pm10}$ & $\textbf{0.84}_{\pm0.02}$ & $\textbf{310}_{\pm3}$ & $\textbf{0.83}_{\pm0.02}$\\
         \bottomrule
    \end{tabular}
    \label{tab:dat}
\end{minipage}
\hfill
\begin{minipage}{.28\textwidth}
        \caption{Average results of VAT trackers in UnrealCV (details in Appendix).}
\vspace{-5pt}
\renewcommand{\arraystretch}{0.94}
\renewcommand{\tabcolsep}{1.8pt}
    \begin{tabular}{l|ccc}
        \toprule
        \multirow{2}{*}{Tracker} & \multicolumn{3}{c}{Average} \\ 
         & $AR\uparrow$ & $EL\uparrow$ & $SR\uparrow$\\ \midrule
         DiMP \cite{dimp} & 204 & 367 & 0.58\\
         SARL \cite{sarl} & 240 & 394 & 0.57\\
         AD-VAT \cite{advat} & 238 & 416 & 0.62\\
         AD-VAT+ \cite{advat+} & 307 & 454 & 0.76\\
         TS \cite{ts} & 312 & 474 & 0.86\\
         RSPT \cite{rspt} & \underline{329} & 478 & 0.92\\
         EVT \cite{evt} & 297 & \underline{490} & \underline{0.95}\\
         FAn \cite{fan} & 237 & 462 & 0.90\\
         FAn+SAM2 \cite{sam2} & 257 & 474 & 0.94\\
         TrackVLA \cite{trackvla} & - & \textbf{500} & \textbf{1.00}\\
         \rowcolor{aliceblue} \textbf{Ours} & \textbf{391} & \textbf{500} & \textbf{1.00}\\
         \bottomrule
    \end{tabular}
    \label{tab:unrealcv}
\end{minipage}
\end{table*}

\section{Experiment}
\label{sec:experiment}

\subsection{Experimental Settings}
\label{sec:exp_setting}

\textbf{Experimental Setup.} We evaluate \methodname in a zero-shot setting on two state-of-the-art VAT benchmarks: UnrealCV~\cite{unrealcv} and DAT~\cite{gcvat}. In UnrealCV, we test \methodname on 3 challenging maps with distractors (Parking Lot (2D), UrbanCity (4D), ComplexRoom (4D)) and 5 single-target maps. In DAT, we perform evaluations under daytime condition across all six scenes, comparing \methodname with within-scene trained models. We further validate \methodname on real-world image datasets including VOT~\cite{vot}, DTB70~\cite{dtb70}, and UAVDT~\cite{uavdt}, and deploy it on a \textit{DJI Tello} drone~\cite{tello} for real-world evaluation. \textit{Implementation details and hyperparameter analysis are provided in the Appendix}.

\begin{figure*}[t]
  \centering
  \begin{subfigure}{0.26\linewidth}
    \includegraphics[width=0.98\columnwidth]{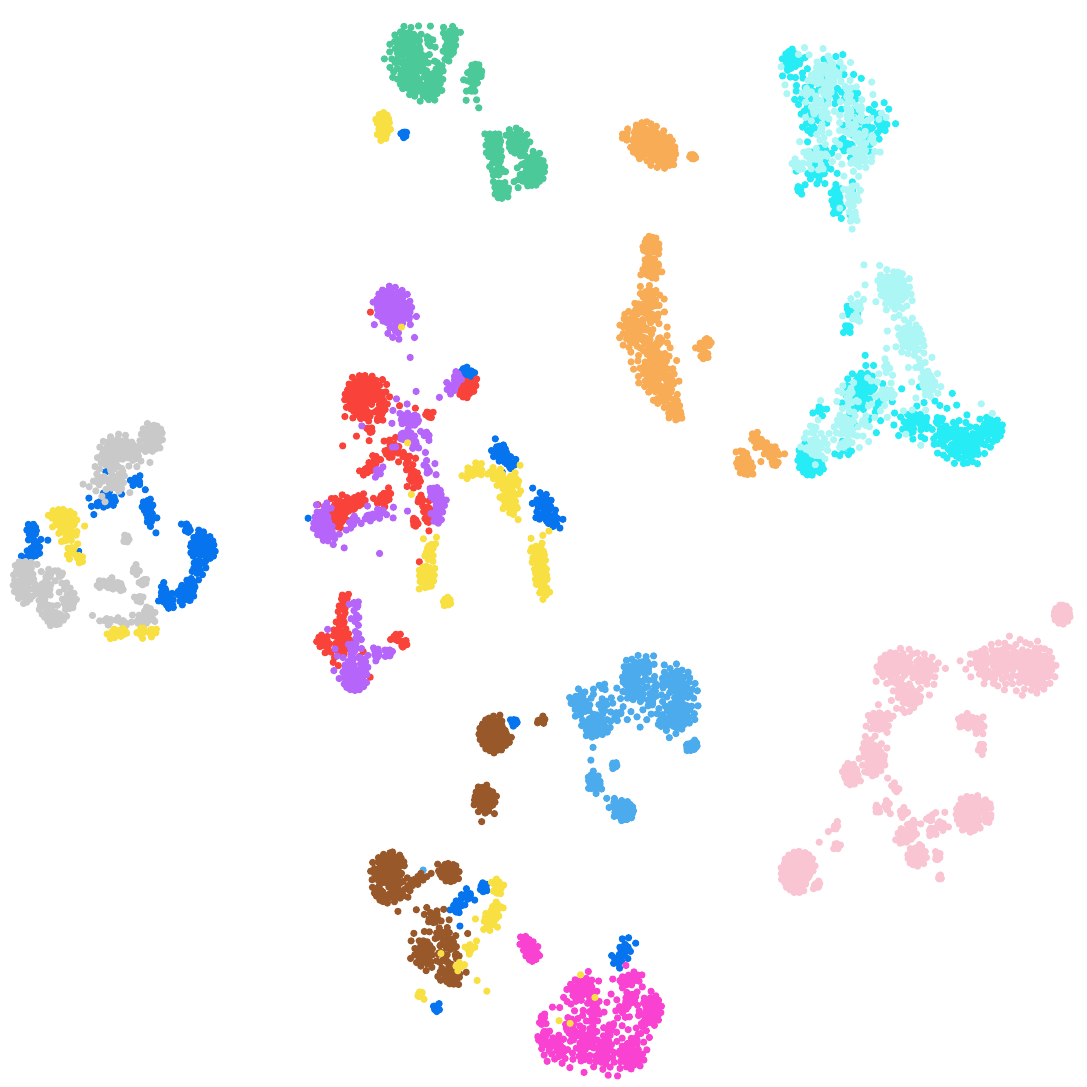}
    \vspace{8pt}
    \caption{t-SNE on prototypes of DINOv3.}
    
  \end{subfigure}
  \hfill
  \begin{subfigure}{0.26\linewidth}
    \includegraphics[width=0.98\columnwidth]{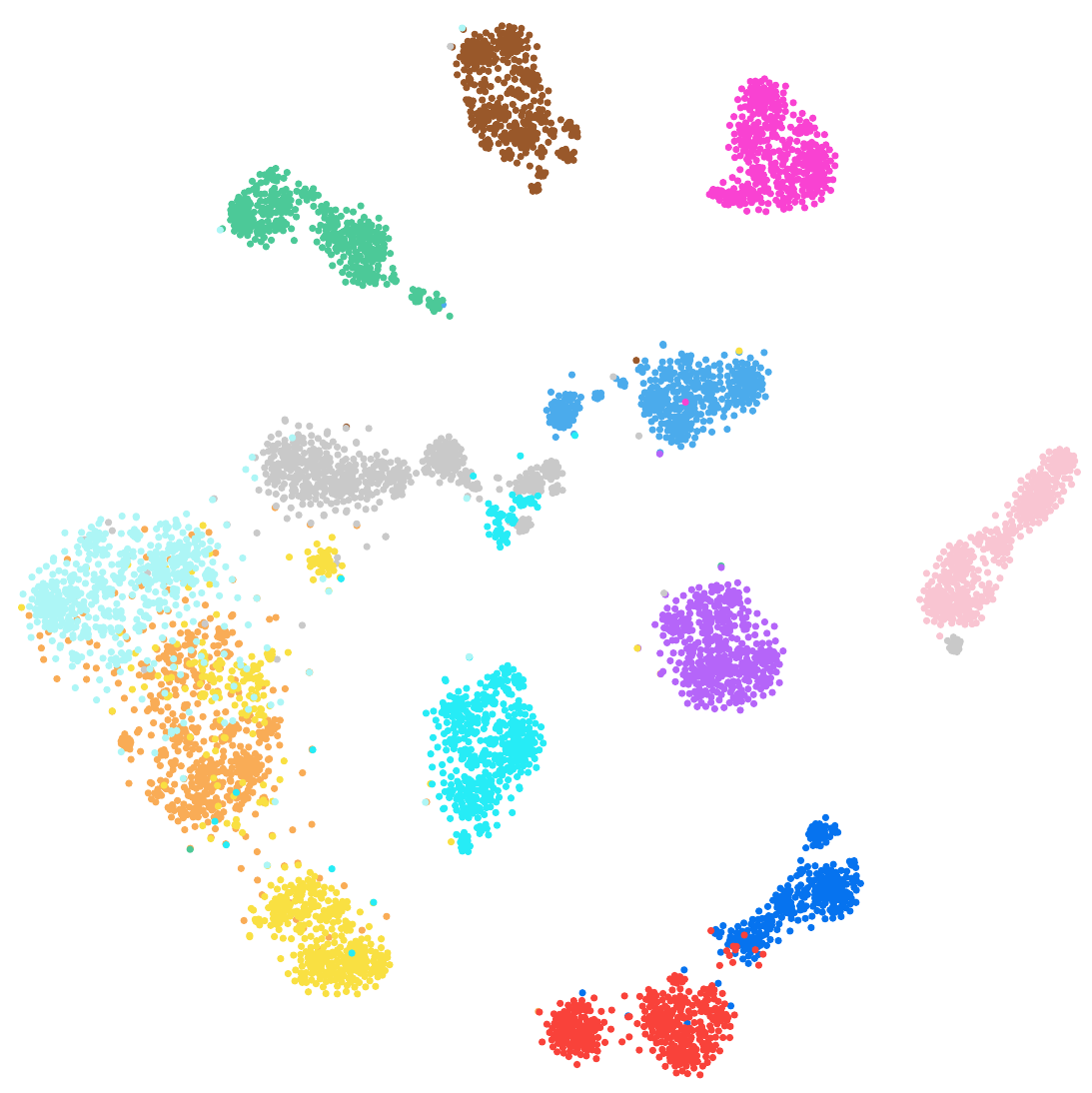}
    \vspace{8pt}
    \caption{t-SNE on prototypes of \methodname.}
  \end{subfigure}
  \hfill
  \begin{subfigure}{0.46\linewidth}
    \includegraphics[width=\columnwidth]{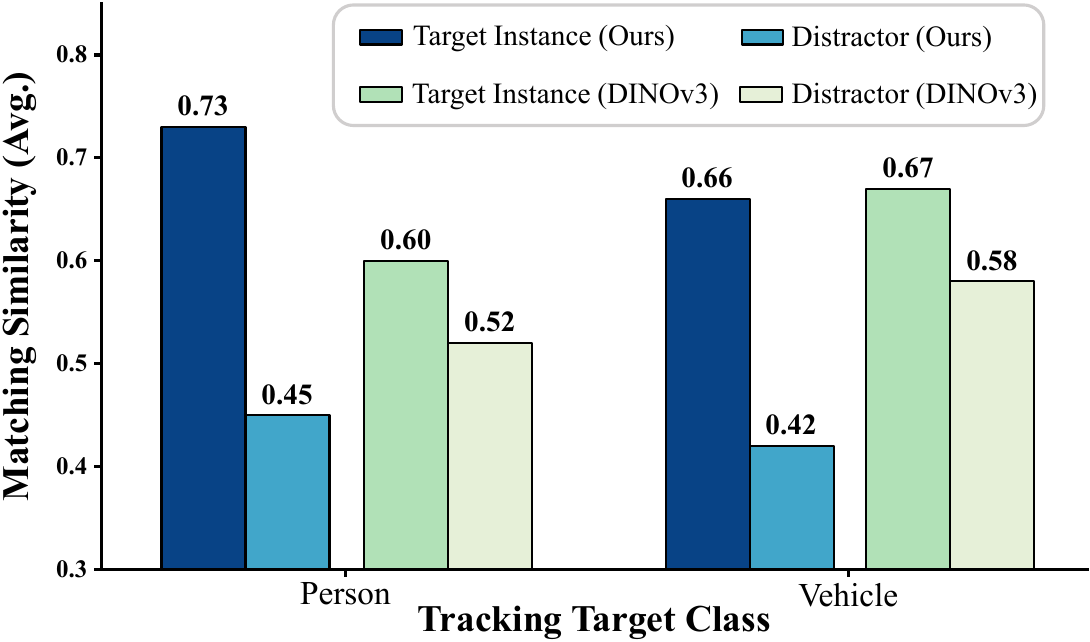}
    \caption{Average cosine similarity of prototype to target and to distractors.}
  \end{subfigure}
  \vspace{-3pt}
  \caption{Ablation study on the instance-aware offline prototype initialization module. In (a) and (b), each point represents a prototype feature extracted from single view of an instance, with all points from the same instance assigned the same color. (c) shows the similarity of the offline initialized prototype with the target instance versus distractors during online tracking.}
  \label{fig:ablation_prototype_init}
\end{figure*}

\begin{table*}[t]
    \small
    \centering
        \caption{Results of ablation experiments on UnrealCV environments containing distractors.}
        \vspace{-5pt}
\renewcommand{\arraystretch}{0.98}
\renewcommand{\tabcolsep}{4pt}
    \begin{tabular}{c|c|ccc|ccc|ccc|ccc}
        \toprule
        \multirow{2}{*}{Module} & \multirow{2}{*}{Setting} & \multicolumn{3}{c|}{Parking Lot (2D)} & \multicolumn{3}{c|}{UrbanCity (4D)} & \multicolumn{3}{c|}{ComplexRoom (4D)} & \multicolumn{3}{c}{Average}\\ 
         &  & $AR$ & $EL$ & $SR$ & $AR$ & $EL$ & $SR$ & $AR$ & $EL$ & $SR$ & $AR$ & $EL$ & $SR$\\ \midrule
         \multirow{2}{*}{\shortstack{Instance-Aware Offline \\ Prototype Initialization}} & w/ DINOv3~\cite{dinov3} Prototype & \underline{384} & \underline{481} & \underline{0.92} & \underline{369} & \underline{443} & \underline{0.84} & \underline{370} & \underline{454} & \underline{0.85} & \underline{374} & \underline{459} & \underline{0.87}\\
         & \cellcolor{aliceblue}\textbf{Ours} & \cellcolor{aliceblue}\textbf{392} & \cellcolor{aliceblue}\textbf{482} & \cellcolor{aliceblue}\textbf{0.93} & \cellcolor{aliceblue}\textbf{385} & \cellcolor{aliceblue}\textbf{486} & \cellcolor{aliceblue}\textbf{0.95} & \cellcolor{aliceblue}\textbf{392} & \cellcolor{aliceblue}\textbf{481} & \cellcolor{aliceblue}\textbf{0.92} & \cellcolor{aliceblue}\textbf{390} & \cellcolor{aliceblue}\textbf{483} & \cellcolor{aliceblue}\textbf{0.93} \\
         \midrule
         \multirow{3}{*}{\shortstack{Online Visual\\Prototype Enhancement}} & w/o Online Enhancement & 352 & 463 & 0.84 & 340 & 469 & 0.77 & 354 & \underline{477} & 0.84 & 349 & 470 & 0.82\\
         & w/ Average Enhancement & \underline{356} & \textbf{485} & \underline{0.90} & \underline{347} & \underline{477} & \underline{0.88} & \underline{359} & \textbf{481} & \underline{0.89} & \underline{354} & \underline{481} & \underline{0.89}\\
         & \cellcolor{aliceblue}\textbf{Ours} & \cellcolor{aliceblue}\textbf{392} & \cellcolor{aliceblue}\underline{482} & \cellcolor{aliceblue}\textbf{0.93} & \cellcolor{aliceblue}\textbf{385} & \cellcolor{aliceblue}\textbf{486} & \cellcolor{aliceblue}\textbf{0.95} & \cellcolor{aliceblue}\textbf{392} & \cellcolor{aliceblue}\textbf{481} & \cellcolor{aliceblue}\textbf{0.92} & \cellcolor{aliceblue}\textbf{390} & \cellcolor{aliceblue}\textbf{483} & \cellcolor{aliceblue}\textbf{0.93} \\
         \midrule
         \multirow{3}{*}{\shortstack{Confidence-Aware\\Kalman Filter}} & w/o Kalman Filter & 349 & 474 & 0.87 & 343 & \underline{459} & 0.88 & 334 & 459 & 0.85 & 342 & 464 & 0.87\\
         & w/ Linear Kalman Filter~\cite{kalman1960new}& \underline{356} & \textbf{488} & \underline{0.88} & \underline{359} & 457 & \underline{0.91} & \underline{346} & \underline{463} & \underline{0.90} & \underline{354} & \underline{469} & \underline{0.90} \\
         & \cellcolor{aliceblue}\textbf{Ours} & \cellcolor{aliceblue}\textbf{392} & \cellcolor{aliceblue}\underline{482} & \cellcolor{aliceblue}\textbf{0.93} & \cellcolor{aliceblue}\textbf{385} & \cellcolor{aliceblue}\textbf{486} & \cellcolor{aliceblue}\textbf{0.95} & \cellcolor{aliceblue}\textbf{392} & \cellcolor{aliceblue}\textbf{481} & \cellcolor{aliceblue}\textbf{0.92} & \cellcolor{aliceblue}\textbf{390} & \cellcolor{aliceblue}\textbf{483} & \cellcolor{aliceblue}\textbf{0.93}\\
         \midrule
         \multirow{4}{*}{\shortstack{Occlusion-Aware\\Trajectory Planner}} & w/o Planning (\textit{i.e.}, PID \cite{pid}) & 345 & \underline{478} & 0.83 & \underline{360} & 470 & \underline{0.88} & 353 & 460 & 0.84 & 353 & 469 & 0.85\\
         & w/ EVT Planning~\cite{evt}& 308 & 472 & 0.82 & 303 & \underline{475} & \underline{0.88} & 360 & \underline{482} & 0.90 & 324 & \underline{476} & 0.87 \\
         & w/o Bounding Box & \underline{368} & 476 & \underline{0.90} & 347 & 456 & 0.86 & \textbf{396} & \textbf{484} & \underline{0.91} & \underline{370} & 472 & \underline{0.89} \\
         & \cellcolor{aliceblue}\textbf{Ours} & \cellcolor{aliceblue}\textbf{392} & \cellcolor{aliceblue}\textbf{482} & \cellcolor{aliceblue}\textbf{0.93} & \cellcolor{aliceblue}\textbf{385} & \cellcolor{aliceblue}\textbf{486} & \cellcolor{aliceblue}\textbf{0.95} & \cellcolor{aliceblue}\underline{392} & 
         \cellcolor{aliceblue}481 & \cellcolor{aliceblue}\textbf{0.92} & \cellcolor{aliceblue}\textbf{390} & \cellcolor{aliceblue}\textbf{483} & \cellcolor{aliceblue}\textbf{0.93} \\
         \bottomrule
    \end{tabular}
    \label{tab:ablation}
    \vspace{-8pt}
\end{table*}

\textbf{Metrics.} In UnrealCV, we use three metrics: \textit{Accumulated Reward (AR)} measures the average reward over 100 episodes. \textit{Episode Length (EL)} denotes the average steps per episode, with early termination if the target remains out of view for more than 50 consecutive steps. \textit{Success Rate (SR)} represents the proportion of success episodes where the tracker successfully follows the target for 500 steps. In DAT, \textit{Cumulative Reward (CR)} measures the average total reward over 40 episodes, and \textit{Tracking Success Rate (TSR)} denotes the success rate within 1,500 steps.

\textbf{Baselines.} We compare \methodname against 12 baselines, grouped into two categories: the RL-based trackers include SARL~\cite{sarl}, AD-VAT~\cite{advat}, AD-VAT+~\cite{advat+}, TS~\cite{ts}, D-VAT~\cite{dvat}, GC-VAT~\cite{gcvat}, and EVT~\cite{evt}. The remaining methods follow a pipeline design: DiMP~\cite{dimp}, RSPT~\cite{rspt}, Follow Anything (FAn)~\cite{fan}, FAn+SAM2~\cite{sam2}, and TrackVLA~\cite{trackvla}. \textit{Related works are detailed in the Appendix}.

\subsection{Comparison Experiments}
\label{sec:exp_comparison}

\textbf{Effectiveness and Efficiency in UnrealCV Environments with Distractors.}
We evaluate \methodname in UnrealCV~\cite{unrealcv} maps with similar distractors. As shown in Tab.~\ref{tab:unrealcv_hard}, \methodname achieves the best performance among all compared methods, including the SOTA method TrackVLA~\cite{trackvla}. Specifically, \methodname outperforms TrackVLA by 2.2\% in average success rate while requiring far less computation. We train \methodname on single RTX 3090 GPU for 15 hours, compared to TrackVLA’s 24×H100 GPUs for the same duration. Furthermore, \methodname has a model size of 584M, significantly smaller than both TrackVLA and EVT. It runs at \textbf{35 FPS} on an RTX 3090 GPU, notably faster than TrackVLA’s 10 FPS on an RTX 4090 GPU, ensuring real-time performance. 

Moreover, we also evaluate \methodname on distractor-free scenes. As shown in Tab.~\ref{tab:unrealcv}, \methodname achieves perfect performance ($SR=1.00$) across all five scenes, maintaining the target in view for the full episode ($EL=500$).  

\textbf{Effectiveness in DAT Environments.} As shown in Tab.~\ref{tab:dat}, \methodname consistently outperforms existing methods. It achieves average improvements of 32.6\% in CR ($242\rightarrow 321$) and 19.4\% in TSR ($0.72\rightarrow 0.86$) over GC-VAT~\cite{gcvat}. The standard deviation of TSR is less than $0.02$, indicating stable tracking performance. Notably, DAT targets are \textit{vehicles}, which are unseen during training, demonstrating the strong zero-shot generalization of \methodname.

\begin{figure*}[t]
\hspace{-3pt}
\begin{minipage}[t]{0.60\linewidth}
    \vspace{-8pt}
    \centering
    \includegraphics[width=\columnwidth]{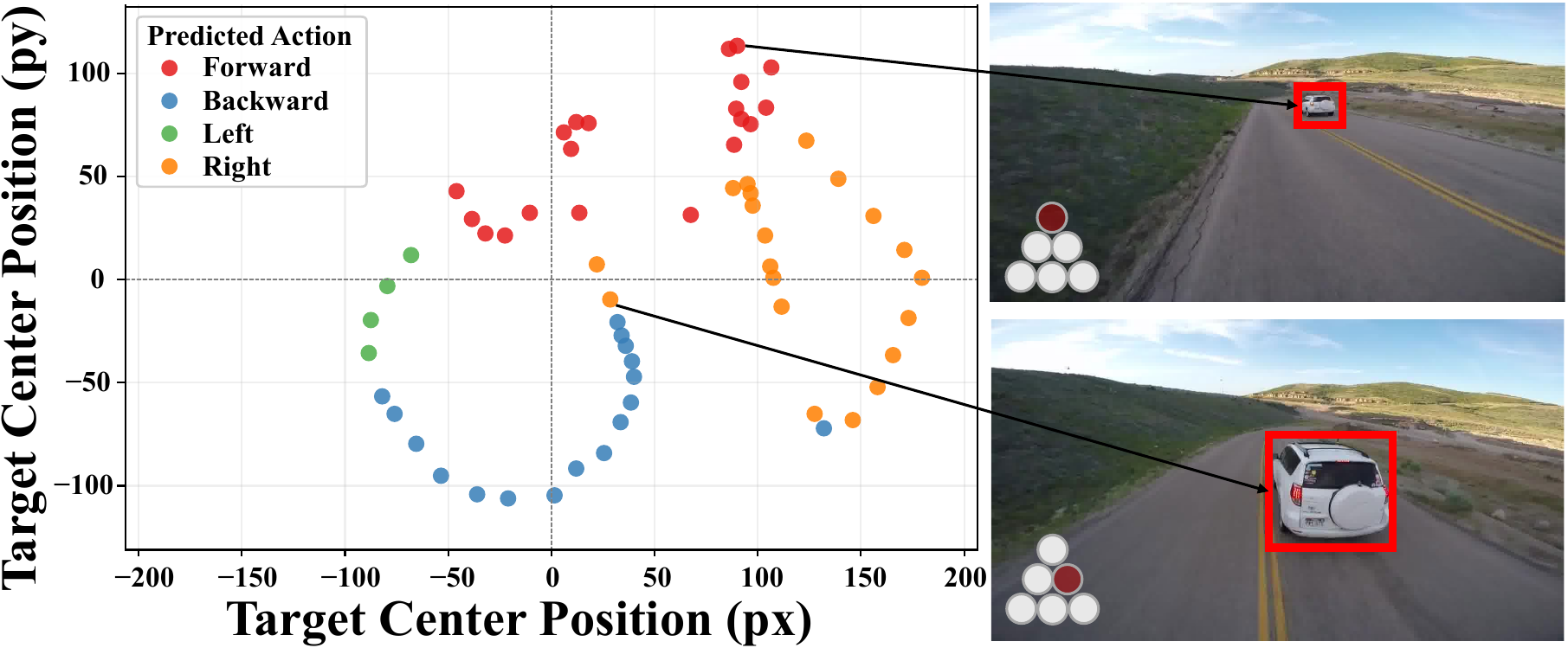}
    \vspace{-16pt}
    \caption{Results on real-world images of the \textit{Car8} video in DTB70 \cite{dtb70} dataset.}
    \label{fig:sim2real_car8}
\end{minipage}
\hfill
\centering
\begin{minipage}[t]{0.38\linewidth}
\vspace{-8pt}
\fontsize{8}{10}\selectfont
\centering
        \captionof{table}{Effectiveness of \methodname on
real-world image evaluation. We select eight videos from each of VOT~\cite{vot}, DTB70~\cite{dtb70} and UAVDT\cite{uavdt} datasets.}
\vspace{-2pt}
\renewcommand{\arraystretch}{1.0}
\renewcommand{\tabcolsep}{1.5pt}
    \begin{tabular}{lccc}
    \toprule[1.5pt]
    \multirow{2}{*}{Tracker} & \multicolumn{3}{c}{Average Correct Action Rate} \\
    \cmidrule(lr){2-4}
    & VOT \cite{vot} & DTB70 \cite{dtb70} & UAVDT \cite{uavdt} \\
    \midrule
    Random & 0.413 & 0.426 & 0.421 \\
    ORTrack~\cite{wu2025ortrack} & 0.661 & 0.781 & 0.879 \\
    ORTrack (w/o bbox) & 0.260 & 0.380 & 0.296 \\
    FAn~\cite{fan} & 0.720 & 0.719 & 0.592 \\
    GC-VAT~\cite{gcvat} & \underline{0.795} & \underline{0.833} & \underline{0.802} \\
    \rowcolor{aliceblue} \textbf{Ours} & \textbf{0.879} & \textbf{0.900} & \textbf{0.945} \\
    \bottomrule
\end{tabular}
    \label{tab:real_image}
\vspace{-6pt}    
\end{minipage}
\end{figure*}

\begin{figure}[t]
  \centering
  \includegraphics[width=\linewidth]{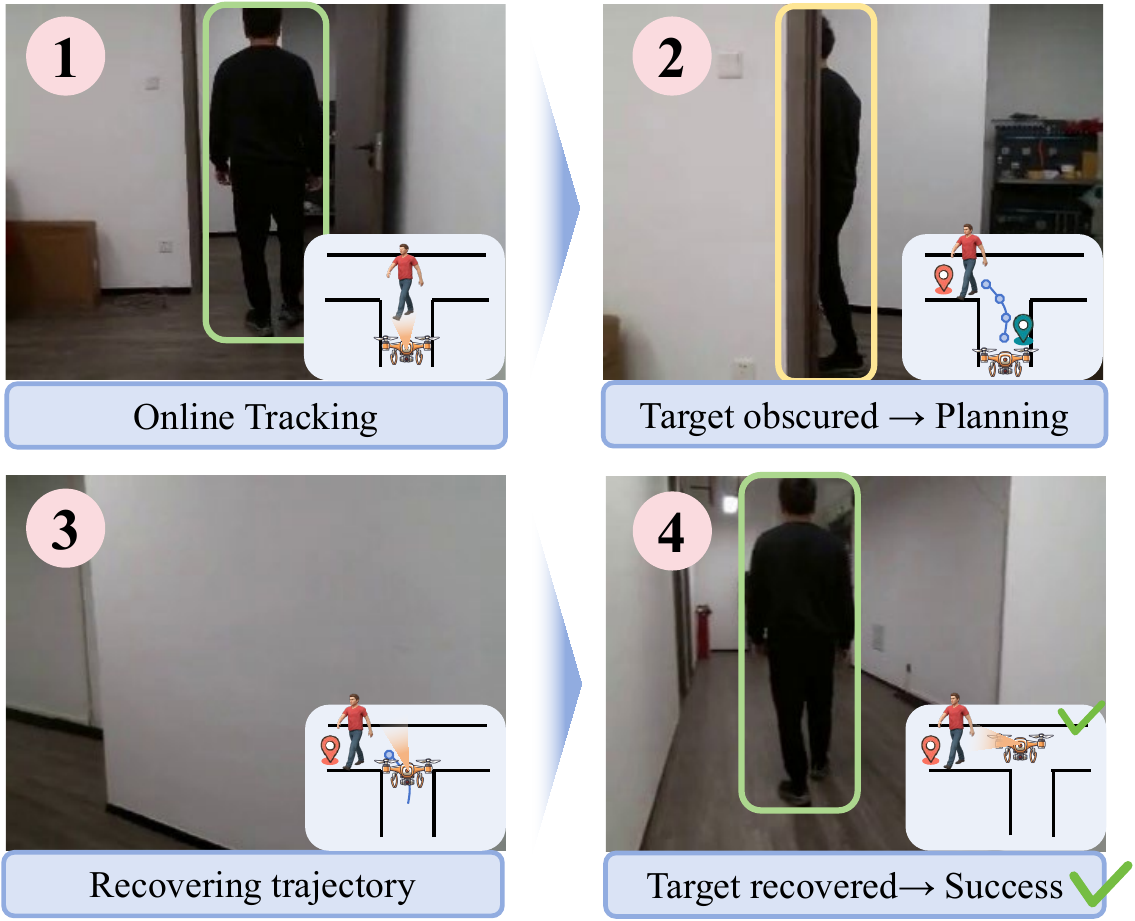}
    \caption{Recovery of prolonged occlusion on a \textit{DJI Tello} drone.}
  \label{fig:real_drone_long_track}
\end{figure}

\subsection{Ablation Experiments}
\label{sec:exp_ablation}

We validate our prototype initialization module (Sec.~\ref{sec:prototype_init}) on video datasets~\cite{personpath22,lasot} through qualitative and quantitative analysis. We then conduct ablation studies on three UnrealCV maps with distractors to evaluate the effectiveness of our prototype initialization, prototype enhancement, confidence-aware Kalman filter, and the planning module.

\textbf{Effectiveness of Offline Prototype Initialization.} We compare our prototype (Sec.~\ref{sec:prototype_init}) with raw DINOv3~\cite{dinov3} features on \textit{video 40} in PersonPath22~\cite{personpath22}. We extract a prototype for each target per frame. As shown in Fig.~\ref{fig:ablation_prototype_init}(a)-(b), our prototypes are well-separated across instances, while DINOv3 features largely overlap. Quantitatively, we compute the average cosine similarity of prototype to targets and to distractors on  \textit{video 40} (for persons) and all \textit{car} videos in LaSOT~\cite{lasot} (for vehicles). \methodname achieves a similarity margin of $0.28$ between person targets and distractors, greatly larger than DINOv3’s $0.08$ (Fig.~\ref{fig:ablation_prototype_init}(c)). When applied to tracking, \methodname improves avg. SR by 6.9\% relative to the DINOv3 baseline, as shown in Tab.~\ref{tab:ablation}(rows 1-2).

\textbf{Effectiveness of Online Prototype Enhancement.} We compare our EMA-based enhancement against average-update and no-update variants. As shown in Tab.~\ref{tab:ablation} (rows 3-5), it outperforms both no-update (+13.4\% avg. SR) and average-update variant (+4.5\% avg. SR), effectively balancing historical knowledge and recent observations.

\textbf{Effectiveness of Confidence-Aware Kalman Filter.} As shown in Tab.~\ref{tab:ablation} (rows 6-8), \methodname improves avg. SR by 6.9\% over the no-filter baseline, and outperforms the linear Kalman filter variant (+3.3\% avg. SR), showing that state estimation helps correct the unreliable bounding boxes.

\textbf{Effectiveness of Planning Module.} We compare our planner with three alternatives: no planning (PID~\cite{pid}), EVT planning~\cite{evt} module based on offline RL, and a planner trained without bounding boxes as input. As shown in Tab.~\ref{tab:ablation} (rows 9-12), \methodname achieves an average SR of 0.93, outperforming EVT by 6.9\% and the variant without bounding box input by 4.5\%. This confirms that bounding box guidance enhances planning robustness, and our \methodname can effectively recover trajectories of occluded targets.

\subsection{Experiments in Real-world Scenarios}
\label{sec:exp_realworld}

\textbf{Effectiveness on real-world images.} To assess \methodname’s transferability to real-world scenarios, we follow the setting of~\cite{gcvat} and perform zero-shot evaluation on 8 videos each from VOT~\cite{vot}, DTB70~\cite{dtb70}, and UAVDT~\cite{uavdt}. Although camera control is unavailable in these videos, we can feed frames into the model and verify whether the predicted action would move the target toward the image center.

Qualitative results on video \textit{Car8} from DTB70 dataset are shown in Fig.~\ref{fig:sim2real_car8}. Each point denotes the target’s location in the image, with its color indicating the predicted action, and arrows showing visual observations. When the target deviates from the image center, \methodname correctly predicts actions to steer it back (e.g., \methodname outputs a rightward control when the target is on the right). Quantitatively, we use Correct Action Rate (CAR), i.e., the action-prediction accuracy, to evaluate the performance. As shown in Tab.~\ref{tab:real_image}, \methodname achieves an avg. CAR of 90.8\%, outperforming GC-VAT~\cite{gcvat} by 12.1\%. \textit{See Appendix for more results.}

Furthermore, to comprehensively evaluate robustness, we compare \methodname with ORTrack~\cite{wu2025ortrack}, the base tracker of the Online Prototype Enhancement Tracker in Sec.~\ref{sec:apt}. It is important to note that ORTrack is a passive visual tracking model that necessitates an initial bounding box as input. As shown in Tab.~\ref{tab:real_image}, ORTrack lags behind our method by 13.4\% on average ($0.774$ \textit{vs.} $0.908$), even though it uses ground-truth boxes for initialization. Moreover, ORTrack collapses when initialized with only a reference image.

\begin{figure}[t]
  \centering
  \begin{subfigure}{0.49\linewidth}
    \includegraphics[width=\textwidth]{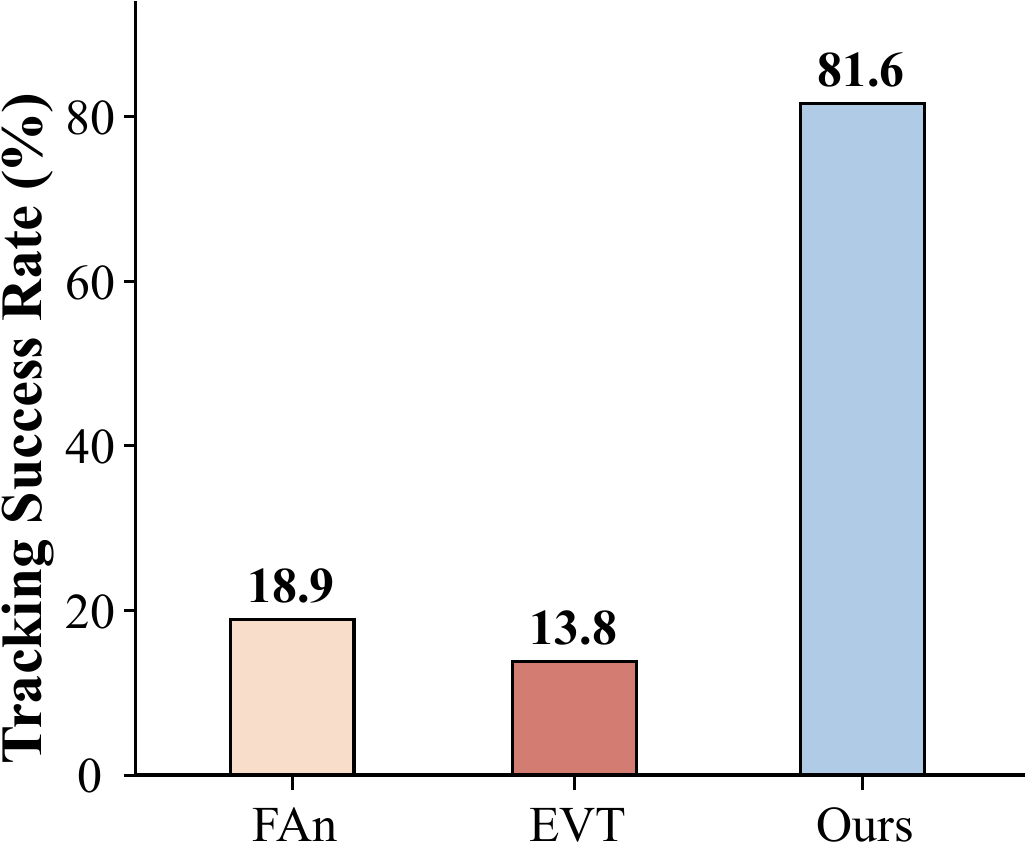}
    \caption{Success rate on real drones.}
    \label{fig:tsr_histogram}
  \end{subfigure}
  \hfill
  \begin{subfigure}{0.49\linewidth}
    \includegraphics[width=\textwidth]{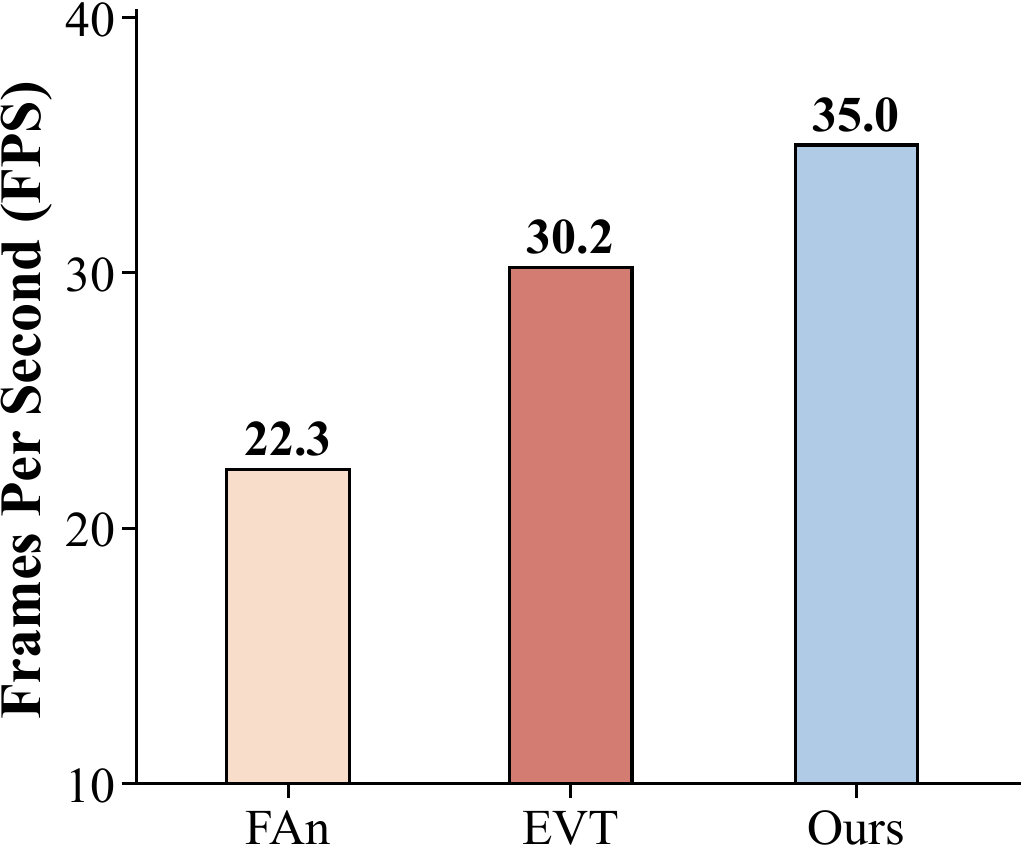}
    \caption{FPS on real drones.}
    \label{fig:fps_histogram}
  \end{subfigure}
  \caption{Experiment results on real drone \textit{DJI Tello}.}
  \label{fig:real_drone_stat}
  \vspace{-10pt}
\end{figure}

\textbf{Effectiveness on real robots.} To assess the real-world applicability of \methodname beyond simulation and image-based benchmarks, we deploy it on a \textit{DJI Tello} drone~\cite{tello}. During operation, the drone transmits H.264 video streams over Wi-Fi to a ground station equipped with an NVIDIA RTX 3090 GPU. The station then decodes the video, runs \methodname to generate control commands, and sends them back to the drone for closed-loop tracking.

We evaluate \methodname in two challenging settings. In scenarios with distractors, \methodname accurately locates and re-identifies the designated target after temporary occlusions, demonstrating robust instance-level performance. Under prolonged occlusion, \methodname actively navigates around obstacles to recover the target (Fig.~\ref{fig:real_drone_long_track}), whereas baseline methods such as EVT~\cite{evt} and FAn~\cite{fan} lose the target permanently. Moreover, We adopt Tracking Success Rate (TSR) to quantify tracking performance. As shown in Fig.~\ref{fig:real_drone_stat}(a), \methodname achieves a TSR of 81.6\%, far surpassing the best baseline (18.9\%). Furthermore, as shown in Fig.~\ref{fig:real_drone_stat}(b), \methodname runs at 35.0 FPS, enabling real-time operation in dynamic scenes. \textit{See Appendix for more results.}

\section{Conclusion}
\label{sec:conclusion}

In this paper, we propose a pipeline VAT method, called \methodname. Specifically, we introduce an instance-aware offline prototype initialization module to mitigate distractor confusion. Then we propose an online tracker that enhances prototypes online and integrates a confidence-aware Kalman filter for stable tracking. Moreover, we develop a planner trained on our new \texttt{Planning-20k} dataset, which plans trajectories to recover occluded targets and generalizes to unseen targets. Experiments on simulators, real-world images, and a drone demonstrate \methodname's state-of-the-art performance and real-time inference.

\section*{Acknowledgements}
This work was partially supported by the Joint Funds of the National Natural Science Foundation of China (Grant No.U24A20327).

{
    \small
    \bibliographystyle{ieeenat_fullname}
    \bibliography{main}
}

\newpage

\appendix

We organize the supplementary materials as follows. Section \ref{sec:related_work} reviews related work on visual active tracking. Section \ref{sec:theo} presents the complete theoretical analysis and proofs of our instance prototype. Section \ref{sec:details} provides implementation details and hyperparameter analysis. Section \ref{sec:exp} reports additional results from both simulation and real-world experiments.

\section{Related Work}
\label{sec:related_work}

\textbf{Visual Active Tracking.} Most VAT methods fall into two categories. RL-based methods~\cite{sarl,advat+,evt,dvat,gcvat} learn end-to-end policies that map observations to actions, enabling low-latency control. However, they rely on sparse rewards that lead to slow convergence in complex scenes. To address this, EVT~\cite{evt} uses offline RL to improve sample efficiency, and GC-VAT~\cite{gcvat} employs curriculum learning~\cite{zhou2024curbench,zhou2022curml} to progressively guide policy optimization. Despite these advances, RL-based trackers still suffer from poor real-world performance due to the sim-to-real gap. 

Pipeline methods~\cite{dimp,fan,rspt,trackvla} decouple tracking into perception and control stages. FAn~\cite{fan} uses foundation models like SAM~\cite{sam1} and DINOv2~\cite{dinov2} to detect the target and feed the predicted bounding box to a PID controller~\cite{pid} for tracking. Although benefiting from strong visual models, these methods struggle to distinguish the target instance from distractors because most visual models are trained for category-level recognition. Besides, the controllers often struggle to recover tracking under occlusions. Recent extensions like TrackVLA~\cite{trackvla} enhance the perception stage by incorporating language instructions based on an LLM, achieving SOTA performance. However, this improvement comes at high computational cost, which degrades performance in high-dynamic scenarios. Drawing inspiration from prototype-based historical representation~\cite{fan,zhou2025pgfa} and diffusion-based planning methods~\cite{trackvla,chi2024diffusionpolicy}, we propose an instance prototype for precise target matching and a diffusion-based planner to recover lost targets under occlusions.

\section{Theoretical Analysis on Instance Prototype}\label{sec:theo}

\textbf{Notations.} For any target instance $T_k$, we define its true feature manifold $M_k$ as the set of all features extracted by $\texttt{Desc}(\cdot)$ under arbitrary imaging conditions. From this manifold, the normalized reference features of $T_k$ are obtained via Eq.~(2) (in the main text) and denoted by the set $F^*_k$. For each reference feature $f^*_k \in F^*_k$, we generate a corresponding set of multi-view augmented features $\{f_{k,i}\}_{i=1}^N$. The mean of the augmented features, denoted $F_{\mathrm{avg},k}$, is given by: $F_{\mathrm{avg},k} = \frac{1}{N}\sum_{i=1}^N f_{k,i}$. The instance-aware prototypes are then derived via Eq.~(3) (in the main text) and form the set $\hat{F}_k$. Finally, for any manifold $M$, we use $\mathbb{E}_{g \sim M}[\cdot]$ to denote expectation over $M$. 

\begin{ass}
\label{ass:global_coverage}
\textbf{(Multi-View Global Coverage)} For any target $T_k$, when the number of augmented views $N$ is sufficiently large, the multi-view augmented features $\{f_{k,i}\}_{i=1}^N$ cover the true feature manifold $M_k$ well enough such that the average squared distance from any point on the manifold to the augmented features is bounded by that of  $f^*_k$:
\begin{equation}
   \mathbb{E}_{g \sim M_k}\left[\frac{1}{N}\sum_{i=1}^N \|f_{k,i} - g\|_2^2\right] \leq \mathbb{E}_{g \sim M_k}\left[\|f^*_k - g\|_2^2\right].
\end{equation}
\end{ass}

Assumption~\ref{ass:global_coverage} captures that multi-view augmentation better covers the global structure of $M_k$, thereby reducing average distance to the manifold.

\begin{ass}
\label{ass:manifold_cohesion_separation}
\textbf{(Manifold Cohesion and Separation)} 
There exist constants $\delta\in(0,1), \ \eta \in (-1, 1)$ such that for any features $g_1, g_2$ belonging to the same target manifold $M_k$, their cosine similarity satisfies $S(g_1, g_2) \geq \delta$ (\textit{intra-manifold cohesion}), while for any features $g_k \in M_k$ and $g_j \in M_j$ from distinct targets $T_k \neq T_j$, the similarity satisfies $S(g_k, g_j) \leq \eta$ (\textit{inter-manifold separation}). 
\end{ass}

Assumption~\ref{ass:manifold_cohesion_separation} ensures cohesive features within the same target and well-separated features across different targets.

\begin{lemma}
\label{lemma:lemma1}
For any target $T_k$, the expected value of the squared Euclidean distance between $F_{avg,k}$ and manifold $M_k$ is no larger than that between $f^*_k$ and manifold $M_k$:
\begin{equation}
    \mathbb{E}_{g \sim M_k}\left[\|F_{avg,k} - g\|_2^2\right] \leq \mathbb{E}_{g \sim M_k}\left[\|f^*_k - g\|_2^2\right].
\end{equation}
\end{lemma}

\begin{proof}
\label{proof:lemma1}
For any fixed $g \in M_k$, the squared Euclidean distance $\|a - g\|_2^2$ is a convex function of $a$. By Jensen’s inequality, for the multi-view augmented features set $\{f_{k,i}\}$:
\begin{equation}
    \left\| \frac{1}{N}\sum_{i=1}^N f_{k,i} - g \right\|_2^2 \leq \frac{1}{N}\sum_{i=1}^N \|f_{k,i} - g\|_2^2,
\end{equation}

\begin{equation}
   \|F_{avg,k} - g\|_2^2 \leq \frac{1}{N}\sum_{i=1}^N \|f_{k,i} - g\|_2^2.
    \label{eq:lemma1_jensen}
\end{equation}
Take the expectation of both sides of Eq.~\eqref{eq:lemma1_jensen} over $g \sim M_k$. For brevity, we denote this expectation by $\mathbb{E}[\cdot]$.
\begin{equation}
   \begin{aligned}
   \mathbb{E}&\left[\|F_{avg,k} - g\|_2^2\right] \leq \mathbb{E}\left[\frac{1}{N}\sum_{i=1}^N \|f_{k,i} - g\|_2^2\right].
   \end{aligned}
   \label{eq:lemma1_expectation}
\end{equation}
By Assumption~\ref{ass:global_coverage}, RHS of Eq.~\eqref{eq:lemma1_expectation} is bounded by:
\begin{equation}
   \mathbb{E}\left[\frac{1}{N}\sum_{i=1}^N \|f_{k,i} - g\|_2^2\right] \leq \mathbb{E}\left[\|f^*_k - g\|_2^2\right].
   \label{eq:lemma1_bound}
\end{equation}
Combining Eq.~\eqref{eq:lemma1_expectation} and Eq.~\eqref{eq:lemma1_bound}:
\begin{equation}
   \mathbb{E}\left[\|F_{avg,k} - g\|_2^2\right] \leq \mathbb{E}\left[\|f^*_k - g\|_2^2\right].
\end{equation}
\end{proof}

\begin{lemma}
\label{lemma:lemma2}
For any target $T_k$, the expected value of the squared Euclidean distance between $\hat{f}_k$ and manifold $M_k$ is smaller than that between $f^*_k$ and manifold $M_k$:
\begin{equation}
    \mathbb{E}_{g \sim M_k}\left[\|\hat{f}_k - g\|_2^2\right] \le \mathbb{E}_{g \sim M_k}\left[\|f^*_k - g\|_2^2\right].
\end{equation}
\end{lemma}

\begin{proof}
\label{proof:lemma2_proof}
For unit vectors $a, b$, the squared Euclidean distance can be rewritten using inner product:
\begin{equation}
   \|a - b\|_2^2 = 2 - 2a \cdot b.
   \label{eq:sqrt_dist}
\end{equation}
Applying Eq.~\eqref{eq:sqrt_dist} to $\hat{f}_k$ and $f^*_k$, Lemma~\ref{lemma:lemma2} reduces to proving:
\begin{equation}
   \mathbb{E}_{g \sim M_k}\left[\hat{f}_k \cdot g\right] \ge \mathbb{E}_{g \sim M_k}\left[f^*_k \cdot g\right].
   \label{eq:lemma2_target}
\end{equation}
For brevity, we denote this expectation by $\mathbb{E}[\cdot]$. Substituting the definition of $\hat{f}_k$ into the left-hand side of Eq.~\eqref{eq:lemma2_target} gives:
\begin{equation}
   \mathbb{E}\left[\hat{f}_k \cdot g\right] = \frac{\mathbb{E}\left[(f^*_k + F_{avg,k}) \cdot g\right]}{\|f^*_k + F_{avg,k}\|_2}.
   \label{eq:lemma2_full_expansion}
\end{equation}
By linearity of expectation, the expectation term expands to:
\begin{equation}
   \mathbb{E}\left[(f^*_k + F_{avg,k}) \cdot g\right] = \mathbb{E}\left[f^*_k \cdot g\right]+ \mathbb{E}\left[F_{avg,k} \cdot g\right]. 
\end{equation}
By Lemma~\ref{lemma:lemma1} and Eq.~\eqref{eq:sqrt_dist}, we have:
\begin{equation}
   2 - 2\mathbb{E}\left[F_{avg,k} \cdot g\right] \leq 2 - 2\mathbb{E}\left[f^*_k \cdot g\right],
   \label{eq:lemma2_expect}
\end{equation}
simplifying gives:
\begin{equation}
   \mathbb{E}\left[F_{avg,k} \cdot g\right] \geq \mathbb{E}\left[f^*_k \cdot g\right] = \mu.
   \label{eq:lemma2_expect_simple}
\end{equation}
Substitute Eq.~\eqref{eq:lemma2_expect_simple} into Eq.~\eqref{eq:lemma2_expect}:
\begin{equation}
   \mathbb{E}\left[(f^*_k + F_{avg,k}) \cdot g\right] \geq 2\mu.
   \label{eq:numerator_lb}
\end{equation}
We now bound $\|f^*_k + F_{avg,k}\|_2$ (denoted as $L$). By Assumption~\ref{ass:manifold_cohesion_separation}, $f^*_k \cdot F_{avg,k} \geq \delta$, and $\|F_{avg,k}\|_2^2 \geq \delta$. Thus,
\begin{equation}
    L \geq \sqrt{1 + 3\delta}.
    \label{eq:l_lb}
\end{equation}
By the triangle inequality, $L\leq \|f^*_k\|_2 + \|F_{avg,k}\|_2$. Thus,
\begin{equation}
   L \leq 2.
    \label{eq:l_ub}
\end{equation}
Substitute Eq.~\eqref{eq:numerator_lb}, Eq.~\eqref{eq:l_lb} and Eq.~\eqref{eq:l_ub} into \eqref{eq:lemma2_full_expansion}:
\begin{equation}
   \mathbb{E}\left[\hat{f}_k \cdot g\right] \geq \mu.
\end{equation}
This proves Eq.~\eqref{eq:lemma2_target}, and thus the lemma.
\end{proof}

\begin{prop}
\label{prop:min_distance}
For any two distinct targets $T_k \neq T_j$, the minimum squared distance between any pair of prototype features $(\hat{f}_k, \hat{f}_j)$ sampled from $\hat{F}_k$ and $\hat{F}_j$ is larger than the minimum squared distance between any pair of reference features $(f^*_k, f^*_j)$ sampled from $F^*_k$ and $F^*_j$:
\begin{equation}
\min_{\hat{f}_k \in \hat{F}_k, \hat{f}_j \in \hat{F}_j} \|\hat{f}_k - \hat{f}_j\|_2^2 \geq \min_{f^*_k \in F^*_k, f^*_j \in F^*_j} \|f^*_k - f^*_j\|_2^2.
\end{equation}
\end{prop}

\begin{proof}
Since all features are unit vectors, proving the inequality is equivalent to proving:
\begin{equation}
\max_{\hat{f}_k, \hat{f}_j} \hat{f}_k \cdot \hat{f}_j \leq \max_{f^*_k, f^*_j} f^*_k \cdot f^*_j.
\end{equation}
By Assumption~\ref{ass:manifold_cohesion_separation} and Lemma~\ref{lemma:lemma2}, there exists a constant $\epsilon > 0$ such that:
\begin{equation}
\max_{\hat{f}_k, \hat{f}_j} \hat{f}_k \cdot \hat{f}_j \leq \eta + \epsilon \leq \max_{f^*_k, f^*_j} f^*_k \cdot f^*_j.
\end{equation}
Therefore:
\begin{equation}
\min_{\hat{f}_k, \hat{f}_j} \|\hat{f}_k - \hat{f}_j\|_2^2 \geq \min_{f^*_k, f^*_j} \|f^*_k - f^*_j\|_2^2.
\end{equation}
\end{proof}

\begin{table*}[t]
    \small
    \centering
        \caption{Comprehensive ablation studies on key hyperparameters. We analyze the EMA momentum $\beta$, Kalman filter parameters ($\gamma$, $\lambda$), and the feature descriptor model size. Models \texttt{B}, \texttt{H+} and \texttt{L} refer to \texttt{dinov3\_vitb16}, \texttt{dinov3\_vith+16} and \texttt{dinov3\_vitl16}, respectively. \textbf{Bold} indicates the best performance, \underline{underline} indicates the second best. Default settings are marked with a \colorbox{aliceblue}{light blue background}.}
    \renewcommand{\arraystretch}{1.0}
    \renewcommand{\tabcolsep}{5.5pt}
    \begin{tabular}{c|c|ccc|ccc|ccc|ccc}
        \toprule
        \multirow{2}{*}{Parameter} & \multirow{2}{*}{Value} & \multicolumn{3}{c|}{Parking Lot (2D)} & \multicolumn{3}{c|}{UrbanCity (4D)} & \multicolumn{3}{c|}{ComplexRoom (4D)} & \multicolumn{3}{c}{Avg.} \\ 
         & & $AR\uparrow$ & $EL\uparrow$ & $SR\uparrow$ & $AR\uparrow$ & $EL\uparrow$ & $SR\uparrow$ & $AR\uparrow$ & $EL\uparrow$ & $SR\uparrow$ & $AR\uparrow$ & $EL\uparrow$ & $SR\uparrow$ \\ \midrule
         
         \multirow{4}{*}{\shortstack{EMA \\ Momentum ($\beta$)}} 
         & 0.6 & 378 & \textbf{485} & \underline{0.91} & \underline{381} & \underline{485} & \underline{0.92} & \textbf{402} & \textbf{487} & \textbf{0.95} & \underline{387} & \textbf{486} & \textbf{0.93} \\
         & 0.7 & 381 & \underline{484} & \underline{0.91} & 378 & 483 & 0.91 & \underline{400} & 480 & \underline{0.93} & 386 & 482 & \underline{0.92} \\
         & \cellcolor{aliceblue}0.8 & \cellcolor{aliceblue}\textbf{392} & \cellcolor{aliceblue}482 & \cellcolor{aliceblue}\textbf{0.93} & \cellcolor{aliceblue}\textbf{385} & \cellcolor{aliceblue}\textbf{486} & \cellcolor{aliceblue}\textbf{0.95} & \cellcolor{aliceblue}392 & \cellcolor{aliceblue}\underline{481} & \cellcolor{aliceblue}0.92 & \cellcolor{aliceblue}\textbf{390} & \cellcolor{aliceblue}\underline{483} & \cellcolor{aliceblue}\textbf{0.93}\\
         & 0.9 & \underline{383} & 472 & 0.90 & 366 & 471 & 0.89 & 393 & 479 & 0.91 & 381 & 474 & 0.90 \\
         \midrule

         \multirow{3}{*}{\shortstack{Kalman Filter \\ Center ($\gamma$)}} 
         & 0.3 & \underline{385} & 469 & \underline{0.91} & \underline{381} & \underline{471} & \underline{0.92} & \underline{397} & \underline{488} & \underline{0.93} & \underline{388} & 476 & \underline{0.92} \\
         & \cellcolor{aliceblue}0.4 & \cellcolor{aliceblue}\textbf{392} & \cellcolor{aliceblue}\underline{482} & \cellcolor{aliceblue}\textbf{0.93} & \cellcolor{aliceblue}\textbf{385} & \cellcolor{aliceblue}\textbf{486} & \cellcolor{aliceblue}\textbf{0.95} & \cellcolor{aliceblue}392 & \cellcolor{aliceblue}{481} & \cellcolor{aliceblue}0.92 & \cellcolor{aliceblue}\textbf{390} & \cellcolor{aliceblue}\underline{483} & \cellcolor{aliceblue}\textbf{0.93}\\
         & 0.5 & 381 & \textbf{490} & \textbf{0.93} & 365 & 467 & 0.90 & \textbf{416} & \textbf{495} & \textbf{0.97} & 387 & \textbf{484} & \textbf{0.93} \\
         \midrule

         \multirow{3}{*}{\shortstack{Kalman Filter \\ Slope ($\lambda$)}} 
         & 13 & \underline{387} & 477 & 0.91 & 361 & 470 & 0.87 & \textbf{418} & \textbf{495} & \textbf{0.96} & \underline{389} & 481 & \underline{0.91} \\
         & \cellcolor{aliceblue}15 & \cellcolor{aliceblue}\textbf{392} & \cellcolor{aliceblue}\textbf{482} & \cellcolor{aliceblue}\textbf{0.93} & \cellcolor{aliceblue}\textbf{385} & \cellcolor{aliceblue}\textbf{486} & \cellcolor{aliceblue}\textbf{0.95} & \cellcolor{aliceblue}392 & \cellcolor{aliceblue}481 & \cellcolor{aliceblue}0.92 & \cellcolor{aliceblue}\textbf{390} & \cellcolor{aliceblue}\textbf{483} & \cellcolor{aliceblue}\textbf{0.93} \\
         & 17 & 375 & 472 & 0.90 & \underline{381} & \underline{483} & \underline{0.94} & \underline{414} & \underline{490} & \underline{0.94} & \textbf{390} & \underline{482} & \textbf{0.93} \\
          \midrule
          
         \multirow{3}{*}{\shortstack{Feature Descriptor \\ (DINOv3) Size}} 
         & \texttt{B} & \underline{389} & \textbf{485} & \textbf{0.93} & \textbf{395} & 479 & 0.89 & \textbf{406} & \textbf{484} & \textbf{0.95} & \textbf{397} & \underline{483} & \underline{0.92} \\
         & \cellcolor{aliceblue}\texttt{L} & \cellcolor{aliceblue}\textbf{392} & \cellcolor{aliceblue}482 & \cellcolor{aliceblue}\textbf{0.93} & \cellcolor{aliceblue}\underline{385} & \cellcolor{aliceblue}\textbf{486} & \cellcolor{aliceblue}\textbf{0.95} & \cellcolor{aliceblue}392 & \cellcolor{aliceblue}\underline{481} & \cellcolor{aliceblue}0.92 & \cellcolor{aliceblue}\underline{390} & \cellcolor{aliceblue}\underline{483} & \cellcolor{aliceblue}\textbf{0.93}\\
         & \texttt{H+} & 388 & \underline{484} & \textbf{0.93} & 377 & \underline{485} & \underline{0.92} & \underline{400} & \textbf{484} & \underline{0.94} & 388 & \textbf{484} & \textbf{0.93} \\
         \bottomrule
    \end{tabular}
    \label{tab:ablation_hyperparameters}
\end{table*}

\begin{table*}[t]
    \small
    \centering
        \caption{Results on UnrealCV benchmark. \textbf{Bold} represents the best while \underline{underline} represents the second.}
\renewcommand{\arraystretch}{0.98}
\renewcommand{\tabcolsep}{5.1pt}
    \begin{tabular}{l|c|ccc|ccc|ccc|ccc|ccc}
        \toprule
        \multirow{2}{*}{Trackers} & \multirow{2}{*}{Publication} & \multicolumn{3}{c|}{SimpleRoom} & \multicolumn{3}{c|}{Parking Lot} & \multicolumn{3}{c|}{UrbanCity} & \multicolumn{3}{c|}{UrbanRoad} & \multicolumn{3}{c}{Snow Village} \\ 
         & & $AR$ & $EL$ & $SR$ & $AR$ & $EL$ & $SR$ & $AR$ & $EL$ & $SR$ & $AR$ & $EL$ & $SR$ & $AR$ & $EL$ & $SR$ \\ \midrule
         DiMP \cite{dimp} & ICCV 2019 & 336 & \textbf{500} & \textbf{1.00} & 166 & 327 & 0.48 & 239 & 401 & 0.66 & 168 & 308 & 0.33 & 110 & 301 & 0.43 \\
         SARL \cite{sarl} & TPAMI 2019 & 368 & \textbf{500} & \textbf{1.00} & 92 & 301 & 0.22 & 331 & 471 & 0.86 & 207 & 378 & 0.48 & 203 & 318 & 0.31 \\
         AD-VAT \cite{advat} & ICLR 2019 & 356 & \textbf{500} & \textbf{1.00} & 86 & 302 & 0.20 & 335 & 484 & 0.88 & 246 & 429 & 0.60 & 169 & 364 & 0.44 \\
         AD-VAT+ \cite{advat+} & TPAMI 2019 & 373 & \textbf{500} & \textbf{1.00} & 267 & 439 & 0.60 & \underline{389} & \underline{497} & \underline{0.94} & 326 & 471 & 0.80 & 182 & 365 & 0.44 \\
         TS \cite{ts} & ICML 2021 & \textbf{412} & \textbf{500} & \textbf{1.00} & 265 & 472 & 0.89 & 341 & 496 & \underline{0.94} & 308 & 480 & 0.84 & 234 & 424 & 0.63 \\
         RSPT \cite{rspt} & AAAI 2023 & \underline{398} & \textbf{500} & \textbf{1.00} & \underline{314} & 480 & 0.80 & 341 & \textbf{500} & \textbf{1.00} & \underline{346} & \textbf{500} & \textbf{1.00} & 248 & 410 & 0.80 \\
         EVT \cite{evt} & ECCV 2024 & 374 & \textbf{500} & \textbf{1.00} & 274 & 484 & 0.92 & 306 & \textbf{500} & \textbf{1.00} & 300 & \underline{496} & \underline{0.96} & 229 & \underline{471} & 0.87 \\
         FAn \cite{fan} & RAL 2024 & 318 & \textbf{500} & \textbf{1.00} & 215 & 481 & \underline{0.96} & 193 & 466 & 0.90 & 152 & 409 & 0.76 & 306 & 456 & 0.90 \\
         FAn+SAM2 \cite{sam2} & ICLR 2025 & 329 & \textbf{500} & \textbf{1.00} & 215 & \underline{491} & \underline{0.96} & 217 & 470 & 0.92 & 207 & 442 & 0.90 & \underline{317} & 465 & \underline{0.94} \\
         TrackVLA \cite{trackvla} & CoRL 2025 & - & \textbf{500} & \textbf{1.00} & - & \textbf{500} & \textbf{1.00} & - & \textbf{500} & \textbf{1.00} & - & \textbf{500} & \textbf{1.00} & - & \textbf{500} & \textbf{1.00} \\
         \rowcolor{aliceblue} \textbf{Ours} & CVPR 2026 & 389 & \textbf{500} & \textbf{1.00} & \textbf{382} & \textbf{500} & \textbf{1.00} & \textbf{390} & \textbf{500} & \textbf{1.00} & \textbf{401} & \textbf{500} & \textbf{1.00} & \textbf{391} & \textbf{500} & \textbf{1.00} \\
         \bottomrule
    \end{tabular}
    \label{tab:unrealcv_full}
\end{table*}

\section{More Details of \methodname} \label{sec:details}

\subsection{Implementation Details}

We train \methodname for 60 epochs using a batch size of 64 on a single RTX 3090 GPU, with the entire training process taking about 15 hours. We extract image features using the pre-trained \texttt{dinov3\_vitl16} (300M parameters)~\cite{dinov3}, which processes 384$\times$384 input images and output 768-dimensional features. Moreover, we generate candidate masks using the \texttt{yoloe-l1l-seg} model~\cite{yoloe}, and select the candidate with the highest similarity exceeding the matching threshold $\eta_s=0.5$ as the target. During tracking, the visual prototype is updated online via an exponential moving average (EMA):
\begin{equation}
    \tilde{\mathbf{f}}' \gets \beta \tilde{\mathbf{f}}' + (1-\beta) \hat{\mathbf{f}}_{\text{tar}},
    \label{eq:EMA}
\end{equation}
where the EMA momentum $\beta$ is set to 0.8. Additionally, we employ a confidence-aware Kalman filter, where the measurement noise covariance $\mathbf{R}_t$ is modeled as a function of the detection confidence $c_t$:
\begin{equation}
    \mathbf{R}_t = \sigma^2(c_t) \mathbf{I}, \quad 
    \sigma^2(c_t) = \frac{1}{1 + e^{\lambda \cdot (c_t - \gamma)}},
    \label{eq:confidence_aware}
\end{equation}
where the hyperparameters are set to $\lambda = 15.0$ and $\gamma = 0.4$. 

\subsection{Hyperparameter Analysis}

We perform ablation experiments to analyze the influence of key hyperparameters in the \methodname method, including the EMA momentum $\beta$ of the online visual prototype enhancement module, the parameters $(\lambda, \gamma)$ of the confidence-aware Kalman filter, and the model size of the feature extractor.

\textbf{EMA Momentum.} We first analyze the effect of the momentum coefficient $\beta$ in the EMA update of the online prototype enhancement module, as shown in Eq.~\eqref{eq:EMA}. We conducted experiments with $\beta$ set to 0.6, 0.7, 0.8, and 0.9. As shown in Tab.~\ref{tab:ablation_hyperparameters} (rows 1-4), \methodname exhibits consistently strong performance across all settings, and our default choice of $\beta=0.8$ achieves the best results,
providing a good balance between historical and current observations.

\textbf{Kalman Filter Parameters.} We then analyze the parameters of the confidence-aware Kalman filter in Eq.~\eqref{eq:confidence_aware}. To evaluate the effect of $\lambda$, we compare the default setting $\lambda=15$ with $\lambda=13$ and $\lambda=17$ (Tab.\ref{tab:ablation_hyperparameters} rows 5-7). We then compare the default $\gamma=0.4$ with $\gamma=0.3$ and $\gamma=0.5$, as shown in Tab.\ref{tab:ablation_hyperparameters} (rows 8-10). The results demonstrate that \methodname is robust to hyperparameter variations.

\begin{figure*}[t]
  \centering
  \includegraphics[width=\linewidth]{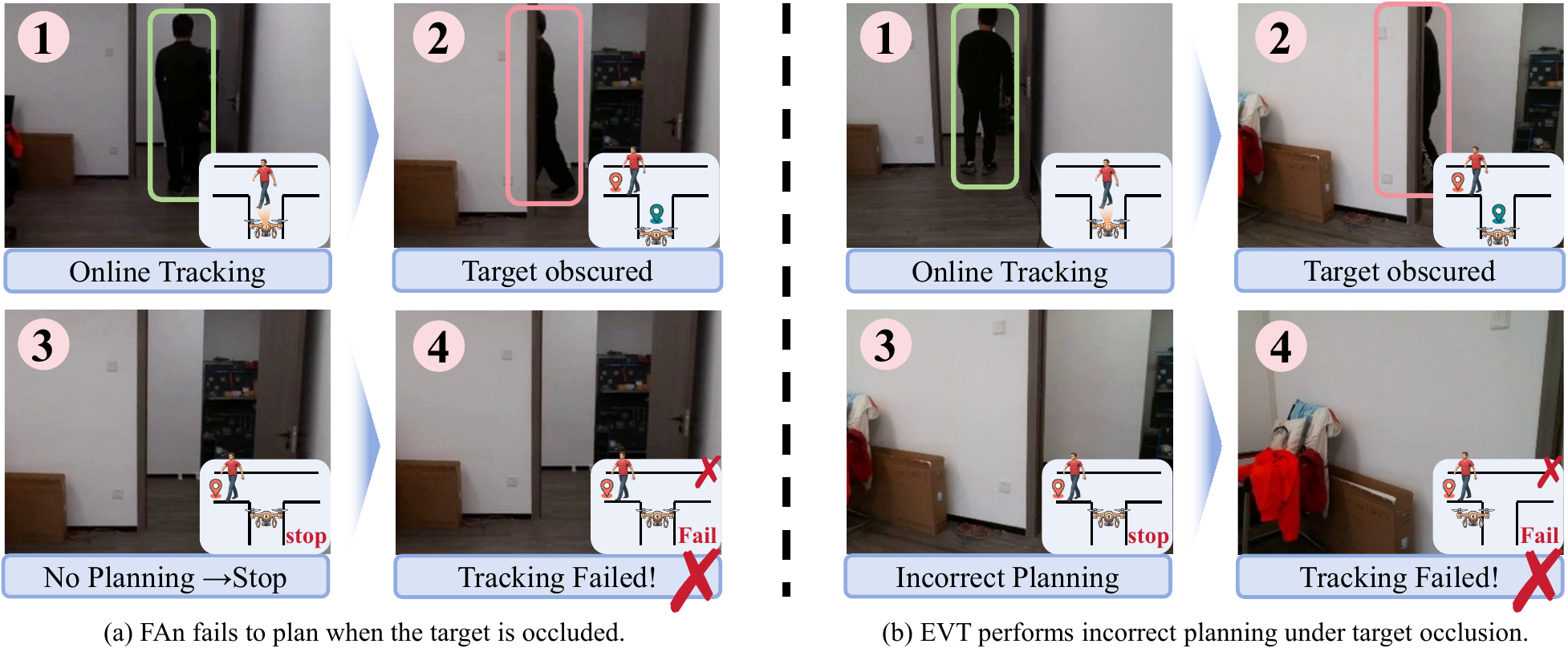}
    \caption{Failure cases of the baseline method FAn~\cite{fan} and EVT~\cite{evt} on real drone \textit{DJI Tello}.}
  \label{fig:real_drone_fan_evt}
\end{figure*}

\begin{figure}[t]
  \centering
  \includegraphics[width=\linewidth]{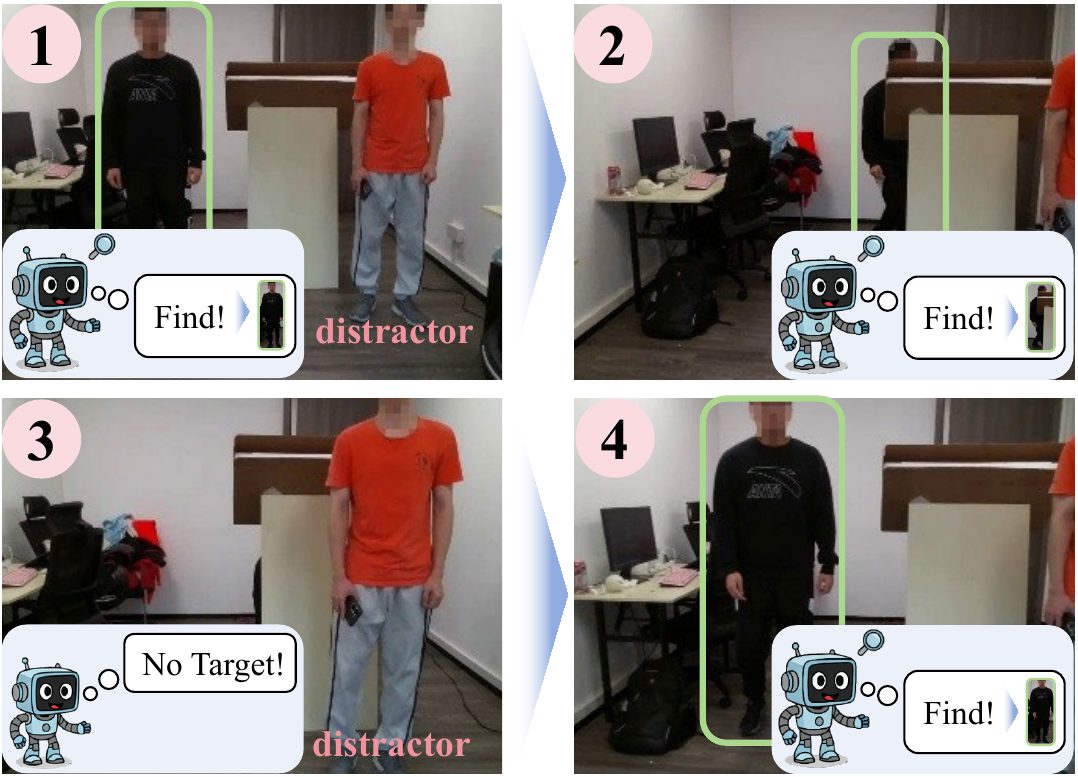}
    \caption{\methodname accurately detects the target against distractors.}
  \label{fig:real_drone_instance}
\end{figure}

\textbf{Feature Extractor Size.} To demonstrate the effectiveness of \methodname across models of different sizes, we evaluate its performance by comparing the default feature extractor \texttt{dinov3\_vitl16} (300M parameters), with \texttt{dinov3\_vitb16} (86M) and \texttt{dinov3\_vith+16} (840M) variants. As shown in Tab.~\ref{tab:ablation_hyperparameters} (Rows 11–13), replacing the 300M extractor with the smaller 86M variant results in only a 1.1\% relative performance drop, indicating that \methodname is highly robust across model sizes.

\section{More Experimental Results}\label{sec:exp}

\subsection{Comparison Experiments}

\textbf{Details of Baselines.} We compare \methodname against 12 baselines: DiMP~\cite{dimp} combines a pre-trained passive tracker with a PID \cite{pid} controller. SARL~\cite{sarl} is an end-to-end RL tracker with a Conv-LSTM backbone. AD-VAT~\cite{advat} and AD-VAT+~\cite{advat+} use an asymmetric dueling RL framework with adversarial learning. RSPT~\cite{rspt} leverages RGB-D input for structure-aware tracking. Cross-modal Teacher-Student (TS)~\cite{ts} employs a cross-modal teacher-student strategy for distraction-robust tracking. EVT~\cite{evt} integrates visual foundation models with offline RL for efficient embodied tracking. Follow Anything (FAn)~\cite{fan} enables open-vocabulary tracking by combining foundation models, and FAn+SAM2 replaces its segmentation module with the latest SAM2 \cite{sam2} model. D-VAT~\cite{dvat} maps RGB observations directly to continuous control signals through reinforcement learning. GC-VAT~\cite{gcvat} designs a goal-centered reward for effective tracking in complex environments. TrackVLA~\cite{trackvla} unifies target recognition and tracking within a single VLA framework built upon an LLM backbone.

\textbf{Additional Results on UnrealCV.} We evaluate all methods on the UnrealCV benchmark for single-target tracking in five distinct virtual scenes: \texttt{SimpleRoom}, \texttt{Parking Lot}, \texttt{UrbanCity}, \texttt{UrbanRoad}, and \texttt{Snow Village}, with detailed results presented in Tab.~\ref{tab:unrealcv_full}. \methodname achieves the highest success rate ($SR=1.00$) across all five scenes, with zero target loss ($EL=500$) in every episode.

\subsection{Experiments in Real-world Scenarios}

To evaluate the real-world applicability and robustness of \methodname, we deploy it on a \textit{DJI Tello} drone. We use the \texttt{DJITelloPy} library~\cite{djitellopy} to capture video streams (at a resolution of $320\times240$) from the drone, transmit them over the network to a ground station equipped with an NVIDIA RTX 3090 GPU, and return the control signals generated by \methodname. The drone operates in velocity control mode, with control commands including linear and angular velocities. The entire pipeline runs at approximately 35 FPS in our experiments. We then evaluate \methodname on two challenging scenarios: long-term tracking with occlusions, and specific instance discrimination against same-category distractors.  

\textbf{Effectiveness in Long-Term Tracking.} \methodname shows superior performance over all baselines in long-term real-robot tracking. It can actively plan collision-free paths and recover the occluded target, as shown in Fig. 6 of the main article. However, as illustrated in Fig.~\ref{fig:real_drone_fan_evt}(a), FAn~\cite{fan} lacks a planning module, causing the robot to stop when the target remains occluded for an extended period, leading to tracking failure. Moreover, as shown in Fig.~\ref{fig:real_drone_fan_evt}(b), EVT~\cite{evt} method, which adopts an offline RL-based planner, still produces incorrect plans under occlusion and fails to navigate around obstacles to recover the target. \textit{Full long-term tracking videos of \methodname, FAn, and EVT methods are provided in the supplementary video.}

\textbf{Effectiveness Under Distractors.} As shown in Fig.~\ref{fig:real_drone_instance}, we evaluate \methodname under distractors, where the person in black is the true target and the one in red is the distractor. During online tracking, \methodname robustly tracks the correct target and remains unaffected when only the distractor is visible, as illustrated in the bottom-left subfigure of Fig.~\ref{fig:real_drone_instance}.

\end{document}